\pdfoutput=1
\documentclass[11pt]{article}

\usepackage[final]{acl}

\usepackage{times}
\usepackage{latexsym}
\usepackage{graphicx}
\usepackage{svg}

\usepackage[T1]{fontenc}
\usepackage{amssymb}
\usepackage{amsthm}
\newtheorem{theorem}{Theorem}
\newtheorem{lemma}{Lemma}

\newtheorem{definition}{Definition}
\newtheorem{innercustomthm}{Theorem}
\newenvironment{customthm}[1]
  {\renewcommand\theinnercustomthm{#1}\innercustomthm}
  {\endinnercustomthm}
\usepackage[utf8]{inputenc}

\usepackage{microtype}

\usepackage{inconsolata}
\usepackage{amsmath}
\DeclareMathOperator*{\argmaxA}{arg\,max}

\title{A Fairness-Driven Method for Learning Human-Compatible Negotiation Strategies}

\author{Ryan Shea \\
  Columbia University, NY \\
  \texttt{rs4235@columbia.edu} \\\And
  Zhou Yu \\
  Columbia University, NY \\
  \texttt{zy2461@columbia.edu} \\}

\begin{document}
\maketitle
\begin{abstract}
Despite recent advancements in AI and NLP, negotiation remains a difficult domain for AI agents. Traditional game theoretic approaches that have worked well for two-player zero-sum games struggle in the context of negotiation due to their inability to learn human-compatible strategies. On the other hand, approaches that only use human data tend to be domain-specific and lack the theoretical guarantees provided by strategies grounded in game theory. Motivated by the notion of fairness as a criterion for optimality in general sum games, we propose a negotiation framework called FDHC which incorporates fairness into both the reward design and search to learn human-compatible negotiation strategies. Our method includes a novel, RL+search technique called LGM-Zero which leverages a pre-trained language model to retrieve human-compatible offers from large action spaces. Our results show that our method is able to achieve more egalitarian negotiation outcomes and improve negotiation quality.
\end{abstract}

\section{Introduction}
Recent advancements in AI and NLP have led researchers to develop techniques to build autonomous agents which can achieve human-level performance in bargaining games such as Deal-or-no-Deal \cite{Sengupta-2021}. These techniques can be separated into two broad categories: game theoretic approaches and data driven approaches. 

Game theoretic approaches to negotiation attempt to build negotiation agents without observing any human data. This class of algorithms is typically applied to two-player zero-sum games which do not require agents to learn human-compatible strategies in order to be successful \cite{alphazero-silver, brown2020rebel, Vinyals2019starcraft}. However, other classes of games such as negotiation require cooperation with humans in order to be successful \cite{cicero}. As a result, traditional game theoretic methods have failed to achieve quality performance in the realm of negotiation \cite{li2023combining}.

Data driven approaches to negotiation learn a human-like strategy directly using data on a specific negotiation domain \cite{verma-etal-2022-chai, lewis-etal-2017-deal, he-etal-2018-decoupling}. Unfortunately, collecting human data is expensive and the strategies in the data may not effectively generalize to other negotiation domains. Furthermore, these methods lack the desirable properties that game theoretic methods offer such as controllability.

% \begin{figure*}[h]
%     \centering
%     \includegraphics[width=\textwidth]{figures/Framework.png}
%     \caption{Outline of our FDHC negotiation framework. Our method consists of decomposing the extensive form Nash bargaining game into a series of depth-limited subgames. At each subgame we calculate the EBS and play a human-like strategy which targets this outcome using a MCTS guided by a LLM and value network.}
%     \label{fig:framewrok}
% \end{figure*}

We propose an approach which bridges game theoretic and data driven approaches called \textbf{F}airness-\textbf{D}riven \textbf{H}uman-\textbf{C}ompatible (FDHC) bargaining. This method is designed to target egalitarian outcomes, specifically the Egalitarian Bargaining Solution (EBS), which we use as a formal notion of fairness \cite{kalai1977-ebs}. We target fair outcomes as prior work has shown that fairness is a key component of human strategies and has also served as a useful notion of optimality in general-sum repeated games \cite{Tossou2020, DiGiovanni2021-self-play, KROLL2014261}. Our fairness-targeting strategy is learned with a novel \textbf{L}LM-\textbf{G}uided \textbf{M}onte Carlo tree search with \textbf{Zero} domain specific training data (LGM-Zero). LGM-Zero leverages the reasoning capabilities of LLM models \cite{kwon2023reward} to extract human-like negotiation offers from large action spaces without the need to collect additional human data. A value model trained via self-play then selects the best offer that the LLM proposes. The LLM and value model are used to guide a Monte Carlo tree search (MCTS) towards the desired outcome given by FDHC (Figure \ref{fig:framewrok}). 

We say that an action/strategy is ``human-compatible'' if a human would take a similar action or apply a similar strategy if placed in the same scenario. For example, in a situation where participants are negotiating over an item worth thousands of dollars, it would not be human-compatible for our model to offer a price down to the granularity of individual cents. Since humans tend not to do that when negotiating over items worth such large amounts. However such an offer would be perfectly valid from a game theoretic standpoint. We use the terms human-like and human-compatible interchangeably.

We implement our method on a common negotiation exercise used in business classes where two students bargain over a used car. We train a model to play as the seller in this scenario. Both the buyer and seller are given private reservation prices for the car that they cannot go beyond during the negotiation. For example, suppose the buyer can't buy the car for above \$11K and the seller cannot sell it for below \$10K. The goal of our method is to reach an agreement at the EBS solution for the game, which in this case corresponds to the midpoint between the two reservation prices (\$10,500 in the example). While we implement our method for this specific exercise, our approach generalizes to any negotiation setting that can be modeled as a Nash bargaining game. This encompasses any game involving surplus division, including multi-party and multi-issue negotiations.

Our final model uses a modular design where negotiation acts are selected according to the FDHC framework. Strategies are then realized in natural language using a LLM such as GPT-3.5 or GPT-4 \cite{openai2023gpt4}. Our contributions can be summarized as follows:

\begin{itemize}
\item We propose novel negotiation framework for bargaining called FDHC. Our framework targets the EBS of Nash bargaining games in an attempt to grant equal gain to both parties.
\item We introduce a RL+search method called LGM-Zero which utilizes a LLM and value network to extract human-compatible offers from large action spaces.
\item Our results show that our method is able to generate more egalitarian outcomes compared to several baselines. Our human evaluation also shows that our model is able to improve negotiation quality while remaining comparable to GPT-4 in human-likeness.
\end{itemize}

\section{Background}
\label{sec:back}
%In this section we define some terms and key concepts in game theory and bargaining theory. Formal definitions for these terms can be found in Appendix \ref{sec:def}.

The \textbf{Nash bargaining game} is a game in which two or more players must divide a surplus between themselves. In the used car example given previously, the surplus would be the difference between the buyer and seller's reservation prices. We use the term \textbf{extensive form Nash bargaining game} to refer to a game in which players can propose divisions of the surplus over the course of a series of time steps. A \textbf{Nash equilibrium} is a game state in which no player can benefit from a unilateral change in strategy. 

%A \textbf{Pareto optimal} outcome is one in which we cannot make any player better off without making another player worse off. 

Bargaining theory makes use of \textbf{axioms} which are rules that describe properties that a bargaining outcome satisfies \cite{Nash}. We make use of the following axioms when analysing the theoretical properties of our method. The axiom of \textbf{symmetry} says that if the players in the bargaining game are indistinguishable based on the description of the game, then they should all receive the same payoff. A \textbf{weak Pareto optimal} solution is one where any change to the outcome will make at least one party no better off. \textbf{Strong monotonicity} states that any increase in the amount of surplus being bargained over should benefit all players involved in the negotiation. Formal definitions can be found in Appendix \ref{sec:def}.

\textbf{Surplus division} is the process of dividing some commodity (often money) among a group of people. A \textbf{utility function} measures the welfare or satisfaction of a negotiator as a function of the amount of surplus they receive. A \textbf{disagreement payoff} is the amount of utility a negotiator receives if the negotiators do not reach an agreement. A \textbf{reservation price} is the minimum amount a seller is willing to sell an item for. The converse holds from the buyer’s perspective. This term is specific to single-issue negotiation.

\textbf{Action space} refers to the set of all valid actions available to an agent as it interacts with an environment. In the context of negotiation, this is the set of actions available to participants in a negotiation. A \textbf{value network} is neural network that takes in a game state and outputs a scalar representing the quality of the state.

\section{Related Work}

Prior work in the field of negotiation has typically been centered on leveraging human data to learn negotiation strategies. These methods involve collecting human-human dialogues for negotiation exercises such as Craigslist bargaining \cite{he-etal-2018-decoupling} or Deal-or-no-Deal \cite{lewis-etal-2017-deal}. This data can then be used to perform supervised learning or offline reinforcement learning on a negotiation model \cite{verma-etal-2022-chai, zhan-etal-2024-lets}. More recent work has focused on examining and enhancing the negotiation capabilities of LLMs \cite{bianchi-etal-2022-language, schneider2023negotiating, fu2023improving, xia2024measuring}. These methods use prompting to create negotiation agents and rely on the zero-shot/few-shot capabilities of LLMs to negotiate.

Data driven methods for negotiation are able to learn human-like negotiation strategies as they directly leverage human data. However they are often overly tailored to one particular domain and have a difficult time generalizing to other scenarios. Furthermore, data driven strategies lack theoretical guarantees such as convergence to a Nash equilibrium is which is a desirable attribute for any negotiation strategy.

Methods grounded in game theory are able to provide the theoretical guarantees that data driven methods lack. As a result they are much more controllable and adapt better to different domains as no additional data collection is needed for training. However, training with no human involvement often results in strategies which are incompatible with human play \cite{cicero}. This has limited work in the area primarily to two-player zero-sum games such as chess where human-compatibility is not needed to ensure robust play \cite{alphazero-silver}. These methods are designed to ensure convergence to a Nash equilibrium, which does not necessarily result in a human-compatible strategy (Section \ref{sec:theoretical_analysis}).  The little work that has attempted to apply game theoretic methods to the negotiation domain tends to ignore the dialogue aspect of negotiation, considering it to be "cheap talk`` \cite{li2023combining}. While the strategic aspect of negotiation can be modeled independently of dialogue, dialogue style has been shown to have a measurable effect on negotiation outcomes \cite{noh2024llms}. Our method is designed to provide theoretical guarantees similar to game theoretic methods while maintaining human-compatibility of data-driven methods by leveraging the reasoning capabilities of LLMs.

\begin{figure*}[h]
    \centering
    \includegraphics[width=\textwidth]{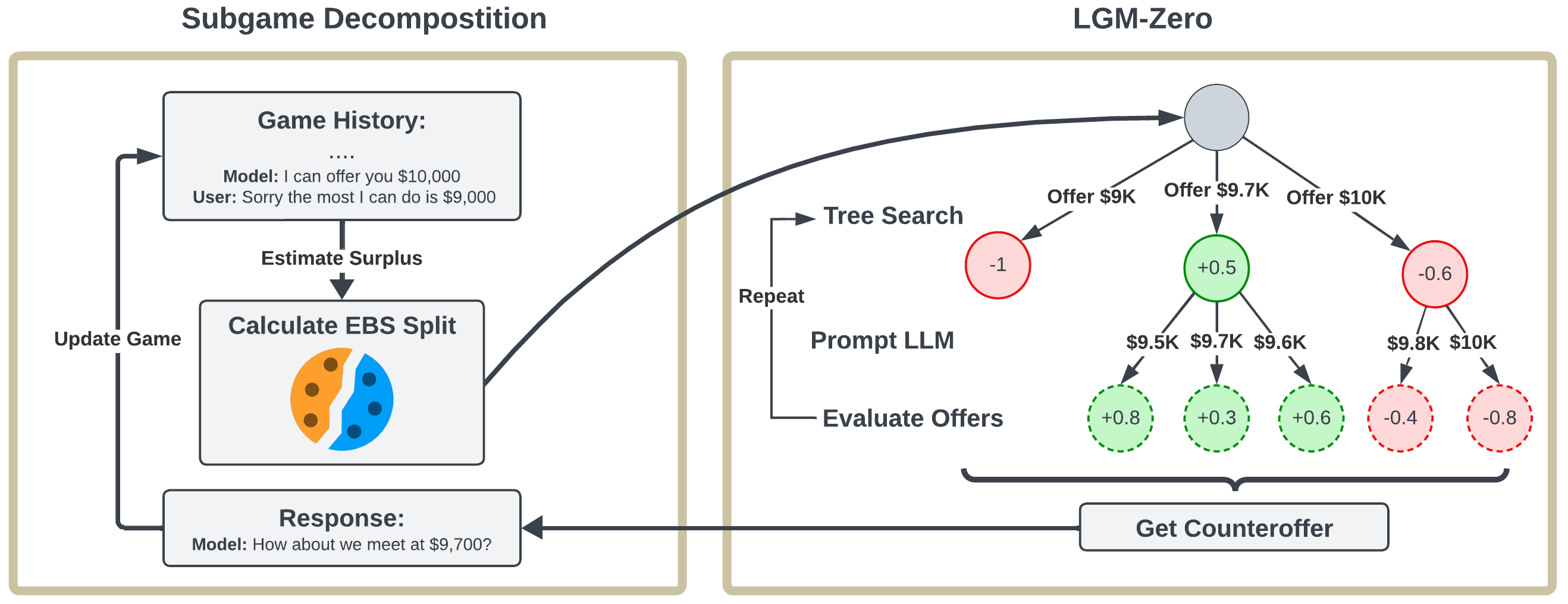}
    \caption{Outline of our FDHC negotiation framework. Our method consists of decomposing the extensive form Nash bargaining game into a series of depth-limited subgames. At each subgame we calculate the EBS and apply a human-like strategy which targets this outcome using a MCTS guided by a LLM and value network.}
    \label{fig:framewrok}
\end{figure*}

\section{Method}
\label{sec:method}

In this section, we describe the FDHC framework which prioritizes egalitarian outcomes. We also describe LGM-Zero, which uses a value model trained with self-play and language model as a policy network. Finally, we outline how we implement our setup for single-issue distributive bargaining.

\subsection{FDHC Negotiation Framework}
FDHC is designed to work within the context of the Nash bargaining game. Specifically, it is designed for an extensive form Nash bargaining game with imperfect information. In this game, players repeatedly request some portion of a surplus, if the sum of their requests at the end of the game is less than or equal to the total surplus then they both receive what they requested, if not they receive a disagreement payoff $d$. FDHC works by decomposing this game in to a series of depth limited subgames. These subgames are identical to the original game, except they may be rooted at any game history and only extend for a limited number of actions in the future. 

Before proceeding to our subgame, we make a guess at the size of the resource pool to be split and our opponent's utility function over these resources. The guess is made based on the history of the game and any initial information we are provided before the game has begun. The specifics of how we do this are domain-dependent and for many games some of the information may be given. For example, in the game Deal-or-no-Deal we know the size of our resource pool but do not know our opponents preferences over the pool. Conversely, in distributive bargaining games we know our opponents preferences but do not know the size of the resource pool.

After making our guess, we root our subgame at the corresponding belief state. This subgame is treated as a perfect-information game and the EBS is calculated as
\begin{equation}
    E(S, d) = \argmaxA_{x \in I(S, d)} (\min_{i \in N}(x_i-d_i)) \nonumber
\end{equation}
where $S$ denotes the bargaining set, $I(S, d)$ is some individually rational payoff set, and $d_i$, $x_i$ are the disagreement payoff and payoff for player $i$, respectively.

Our model then applies a strategy which targets this solution using LGM-Zero, described in the next section. We make moves according to this strategy until the subgame concludes. This can be as short as one move or as long as the entire game depending on our choice of subgame length. We then update our guess for the resource pool and utility function based on our opponents moves and transition to the next subgame. This process is repeated until the game concludes.

\subsection{LGM-Zero}
\label{sec:search_method}

Now we describe LGM-Zero, which uses a MCTS guided by a LLM and value network to perform negotiation actions. Under our setup the only model we train is our value network which is trained via self-play. We first describe how our method behaves during inference time and then describe the process we use to train our value model.

\subsubsection{Inference}

Given the action history of a negotiation our algorithm searches for the best response by repeatedly performing selection, expansion, and backpropagation. We describe these stages next.

\textbf{Selection} During this stage we traverse the game tree by selecting the action, $a$, with the highest upper confidence bound for its Q-value \cite{alphazero-silver}, calculated as
\begin{equation}
    U(s, a) = Q(s, a) + c_p * \frac{\sqrt{\sum_b N(s,b)}}{1 + N(s, a)} \nonumber
\end{equation}
where $s$ is the current game state, $c_p$ is a hyperparameter which controls the degree of exploration, and $N(s,a)$ denotes the number of times we have taken the action previously. The selection process is repeated until we reach a leaf node, which is a defined as a state whose children have not been explored yet \cite{wiechowski2021MonteCT}.

\textbf{Expansion} In the expansion phase we feed a LLM a prompt to suggest five good actions given the current game state. The prompt used to generate actions must be engineered specifically for the negotiation scenario the search is being applied to. We treat all the actions as having equal probability under the model and all other actions at the current state to have a probability of zero. If one of the actions results in a terminal state its value is set to the reward returned by the state, otherwise it's set to the output of our value model. These values are propagated back up the tree according to the next step.

\textbf{Backpropagation} After expansion is concluded we update each node along the search path by incrementing $N(s,a)$ by one for each action taken during the search. We also update the Q-values along the search path as
\begin{equation}
    Q(s,a) \leftarrow Q(s,a) + \frac{v(s)}{N(s,a)} \nonumber
\end{equation}
where $v(s)$ is the value of the state we evaluated, given either by our value model or the actual reward value depending on if the state is terminal.

We repeat this search for $n$ iterations then make a move based on which child of the current state has the highest Q-value.

\subsubsection{Training}

Our value model is trained using a method similar to fictitious self-play \cite{Heinrich2015FictitiousSI}. Fictitious self-play is an iterative method for computing an approximate Nash equilibrium. This is done by performing self-play with a mixed strategy that chooses between playing a best response to our opponent's strategy and the average strategy for the current player. The fictitious self-play set up traditionally learns the best response strategy with a deep Q-network \cite{mnih2013playingatarideepreinforcement} and the average strategy via supervised learning.

Prior work has suggested augmenting the best response step with a MCTS \cite{Zhang2019fictitious-mcts}. We adopt this approach and use the same search process we use during inference time for the best response strategy. To apply an average strategy we once again leverage a LLM but instead of ranking the offers with our value network, we simply have it suggest one move.

The training data for the value model consists of game states and outcomes for the depth limited subgames described in the previous section. The reward for each subgame, from the perspective of player one, is given by 
\begin{equation}
    v(s) =
    \begin{cases}
         \min_{i \in N}(x_i-d_i) & \text{if}\ x_1 \geq E(S,d) \\
         -\min_{i \in N}(x_i-d_i) & \text{if}\ x_1 < E(S,d)
    \end{cases} \nonumber
\end{equation}
This reward says that if the payoff for player one is greater than or equal to the EBS of the subgame, then the reward is simply the EBS score for the game state. If the player's payoff is less than the EBS then they receive the negative EBS score for the game state. Our reward design reflects the fact that humans care about a combination of fairness and their own utility. Therefore our model will target the EBS solution (which has the max reward value) while also preferring outcomes that result in better payouts for itself. 

\subsection{Implementation}
We implement our proposed method for a single-issue distributive bargaining exercise. This exercise involves two parties negotiating over the price of a used car and is used in graduate-level business classes (see Appendix \ref{sec:appendix} for the scenario). The buyer and seller are both given private reservation prices which they cannot go beyond during the negotiation. In our scenario the seller cannot go below a price of \$12,500 and the buyer cannot go above \$13,500. The difference between the reservation prices is the surplus for the game. Our model is trained to act as the seller in this scenario. We assume that our opponents are risk neutral and have a disagreement payoff of \$100. This disagreement payoff is chosen based on experimental results which show that inefficient outcomes, such as disagreements, are common in negotiation \cite{FELTOVICH2011554, Ellingsen2004} suggesting that many humans may prefer to not reach a deal instead of agreeing to a outcome which gives little payoff.

Our final design uses a modular framework where the negotiation acts are extracted from user responses using GPT-4. Our schema uses four acts: no\_counteroffer, counteroffer, accept, reject. These acts are translated into our game state which consists of the offer history for the game (ex. [1500, 1100, 1450, 1200,\ldots]). If the user rejects an offer or gives no counter offer then we assume that they are maintaining their previous offer. If they accept the offer then we assume that their offer is equal to FDHC’s offer. Then a counteroffer is generated using FDHC and LGM-Zero. This offer is realized in natural language by prompting GPT-3.5 to generate a response incorporating the action.

We use GPT-3.5 as our LLM policy network and a transformer with 50 encoder layers and 50 decoder layers as our value network. Our initial subgame is rooted at the belief state for the surplus corresponding to the price range given in the initial description of the car. After the subgame concludes, our new guess for the surplus is equal to the difference between our current offer and the maximum between our opponent's offer and our reservation price. At the final turn of the negotiation we offer \$100 (our disagreement payoff) above our reservation price or accept our opponents offer if it's above this price. Additional implementation details can be found in Appendix \ref{sec:implementation}.

\section{Theoretical Analysis}
\label{sec:theoretical_analysis}
In this section we analyse the theoretical properties of the FDHC framework. Our analysis assumes that, when needed, we can manipulate our LLM policy so that one of the offers it outputs is equivalent to the EBS.

We can ensure that our framework will result in a Nash equilibrium under fairly mild assumptions. We need to assume that the bargaining game is conducted during a finite number of time steps and that the number of steps is known to both players. This gives us the result in Theorem \ref{thm:fdhc_ne}.

\begin{theorem}
\label{thm:fdhc_ne}
Let $t_n$ denote the FDHC's final turn in the negotiation, let $\alpha$ denote the outcome proposed at $t_{n-1}$, and let EBS($x$) denote the EBS value for some outcome $x$. Setting FDHC's estimate of $S = \argmaxA(\text{EBS}(\alpha), \text{EBS}(d))$ at $t_n$ will result in a Nash equilibrium outcome.
\end{theorem}

The proof for this result is straightforward and is presented in Appendix \ref{sec:proofs}. What this theorem says is that we can adjust our surplus estimate so that at its final turn, FDHC will either concede all of the surplus to its opponent(s) or accept the opponents' offer, so long as the offer is larger than its disagreement payoff. This strategy will ensure that the negotiation ends in a deal that splits the entire surplus if one is feasible. Since any deal which splits the whole surplus results in a Nash equilibrium (Appendix \ref{sec:proofs}), our method will give a Nash equilibrium outcome.

This result also demonstrates that convergence to a Nash equilibrium alone is not enough to ensure a robust negotiation agent. However, this does give our agent a baseline level of quality as it means that we are guaranteed to reach a deal if one is feasible. This is in contrast to data-driven methods which provide no such guarantees and have been empirically shown to give inefficient negotiation outcomes \cite{bianchi-etal-2022-language}. We consider a negotiation agent to be robust if it can perform well against a variety of strategies.

Under stronger assumptions, we can guarantee that FDHC will converge to the EBS in expectation. First, we need to assume that both FDHC and the other negotiators have a method to obtain an unbiased estimate of the true surplus value. We also need to make some assumptions about the bargaining outcome induced by our opponents, specifically we make use of the bargaining axioms given in Section \ref{sec:back}.

\begin{theorem}
\label{thm:fdhc_ebs}
Let $F(S,d)$ denote the bargaining outcome targeted by FDHC's opponents. If $F(S,d)$ satisfies the axioms of symmetry, weak Pareto optimality, and strong monotonicity then the expected outcome of the Nash bargaining game will be $E(s,d)$.

\end{theorem}

Our proof of this theorem follows \citealp{CONLEY1991_ebs_proof} and is presented in Appendix \ref{sec:proofs}. 

%Note that our analysis of the EBS outcome is given within the framework of cooperative bargaining. This allows us to abstract away the specifics of the bargaining procedure and instead focus on outcomes of a bargaining game. This is beneficial due to being highly generalizable however  

\section{Experiments}
We test the effectiveness of our method using both automatic and human evaluations. Our results show that our method is able to generate fairer outcomes than existing negotiation baselines. Our human evaluation also shows that our method improves perceived negotiation quality while maintaining the same level of human-like negotiation as GPT-4.

\subsection{Baselines}
We test our method against six negotiation baselines described below.

\textbf{Supervised Learning (SL)} We use the SL agent described in \citealp{he-etal-2018-decoupling} as our first baseline. This method uses the Craigslist bargaining dataset \cite{he-etal-2018-decoupling} to train a negotiation agent via supervised learning. More details on this baseline can be found in Appendix \ref{sec:implementation}.

% The dialogue manager is a classifier which selects one of nine coarse negotiation actions to take at each step of the negotiation. The action is then realized in natural language using a dialogue model. 

\textbf{Offline RL} Our second baseline is based on the CHAI method given in \citealp{verma-etal-2022-chai}. This method uses the Craigslist bargaining dataset to train a negotiation agent with offline Q-learning instead of SL. Implementation details for this method can be found in Appendix \ref{sec:implementation}.

\textbf{GPT-3.5 and GPT-4} We setup GPT-3.5 and GPT-4 for negotiation by prompting them with a summarized version of the scenario in Appendix \ref{sec:appendix}. We find that giving them the full scenario results in oversharing information. We also explicitly tell the model not to reveal its reservation price.

\textbf{GPT-4 Self-Play} We include another baseline using the method described in \citealp{fu2023improving}. This method uses self-play to generate a prompt to improve the negotiation performance of GPT-4. Additional details can be found in Appendix \ref{sec:implementation}.

%First, a GPT-4 seller engages in a negotiation with a GPT-4 buyer until a deal is reached. Then the conversation is fed into another instance of GPT-4 which is asked to give some suggestions on how the seller can improve its negotiation performance. The previous negotiation, along with the suggestions are incorporated into the prompt for the next negotiation.

\textbf{Vicuna-13b} Our final baseline consists of a 13b parameter Vicuna model fine-tuned using synthetic data generated from GPT-4. We generate 108 negotiation transcripts for various negotiation scenarios and use them to train the Vicuna model with the goal of distilling a high quality strategy.

\begin{figure*}[h]
    \centering
    \includegraphics[width=\textwidth]{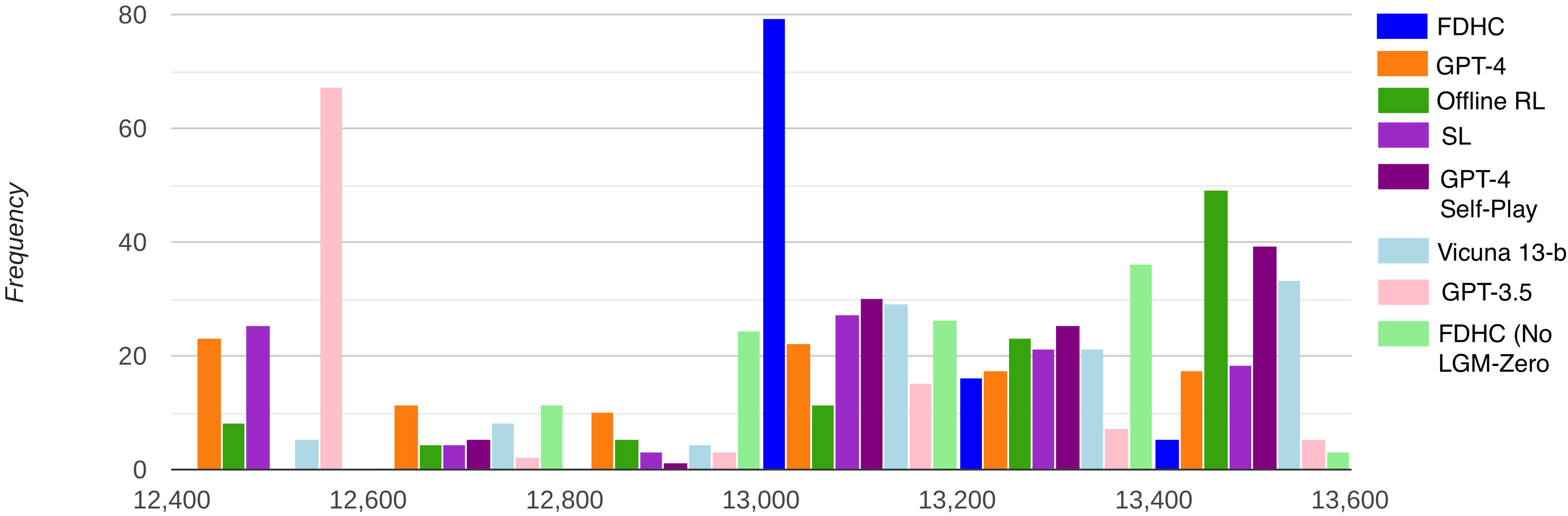}
    \caption{ Binned deal price frequencies of 100 negotiations between our baselines and a GPT-4 buyer. Our goal is to achieve deal prices that minimize the difference in payoff between the buyer and seller. In our scenario this amount is minimized at a deal price of \$13,000.}
    \label{fig:auto_eval}
\end{figure*}

\begin{table*}[h]
\centering
\begin{tabular}{lccc}
\hline
\textbf{Model (Seller)} & \textbf{Average Deal Price} & \textbf{Average Fairness$\uparrow$} & \textbf{Median Fairness$\uparrow$}\\
\hline
GPT-3.5 & \$12,644 (357) & -0.88 (0.49) & -1.0\\
Offline RL & \$13,224 (308) & -0.68 (0.34) & -0.8\\
SL & \$12,978 (368) & -0.59 (0.44) & -0.6\\ 
GPT-4 & \$12,968 (346) & -0.57 (0.39) &	-0.5\\
GPT-4 Self-Play & \$13,242 (240) & -0.54 (0.41) & -0.5\\
Vicuna-13b & \$13,156 (293)  & -0.53 (0.40) &	-0.5\\
FDHC (No LGM-Zero) & \$13,042 (211) & -0.36 (0.23) & -0.4\\
FDHC & \$13,062 (128) & \textbf{-0.12 (0.26)}$^*$ & \textbf{0.0} \\
\hline
\end{tabular}
\caption{\label{tab:auto_eval}
Results of our simulated negotiation evaluation. We consider optimal outcomes to be those which achieve the highest values for fairness, which we define as outcomes which minimize the payoff difference between our buyer and seller. We also report the average deal price between the buyer and seller, standard deviations are shown in parentheses. Statistically significant improvements (independent two-sample t-test, $p < 0.05$) over the baselines are marked with $^*$.
}
\end{table*}

\subsection{Automatic Evaluation}
\label{sec:auto_eval}

For our automatic evaluation we conduct 100 simulated negotiations between our baselines and a GPT-4 buyer. We consider optimal outcomes in these negotiations to be ones which achieve the highest values for fairness, as we consider this outcome to be the most human-compatible. For our evaluations, fairness is defined as the difference in payoff between the buyer and seller. 

The results of our evaluation are shown in Table \ref{tab:auto_eval} and Figure \ref{fig:auto_eval}. Our results show that FDHC is able to achieve much higher values of fairness compared to our baselines. More than 50\% of the deals reached in our negotiations achieve a payoff difference of zero. 

We find that our LLM-based baselines generally perform better than the methods trained using domain-specific data (SL and offline RL). This may be due to the fact that there is a slight mismatch between the negotiation scenarios described in the Craigslist bargaining dataset and the one in our experiment (dataset details can be found in Appendix \ref{sec:implementation}). The scenario in our experiment gives negotiators explicit reservation prices which we use to calculate utilities. However, in the Craigslist bargaining scenarios no reservation prices are given and instead must be inferred. \citealp{he-etal-2018-decoupling} provide a method for inferring these prices which we use here, however the lack of explicit reservation prices may still be harming negotiation performance. We believe this highlights an inherent weakness of data driven bargaining methods as new data must be collected in order to ensure high quality performance in new negotiation domains.

Our LLM-based baselines all perform similarly in terms of fairness, with no statistically significant differences between the outcomes. Given the lack of differences as well as the fact that the average GPT-4 deal price is the most egalitarian out of all these models, we choose to use the GPT-4 baseline for comparison in our human evaluation.

\begin{figure*}[h]
    \centering
    \includegraphics[width=\textwidth]{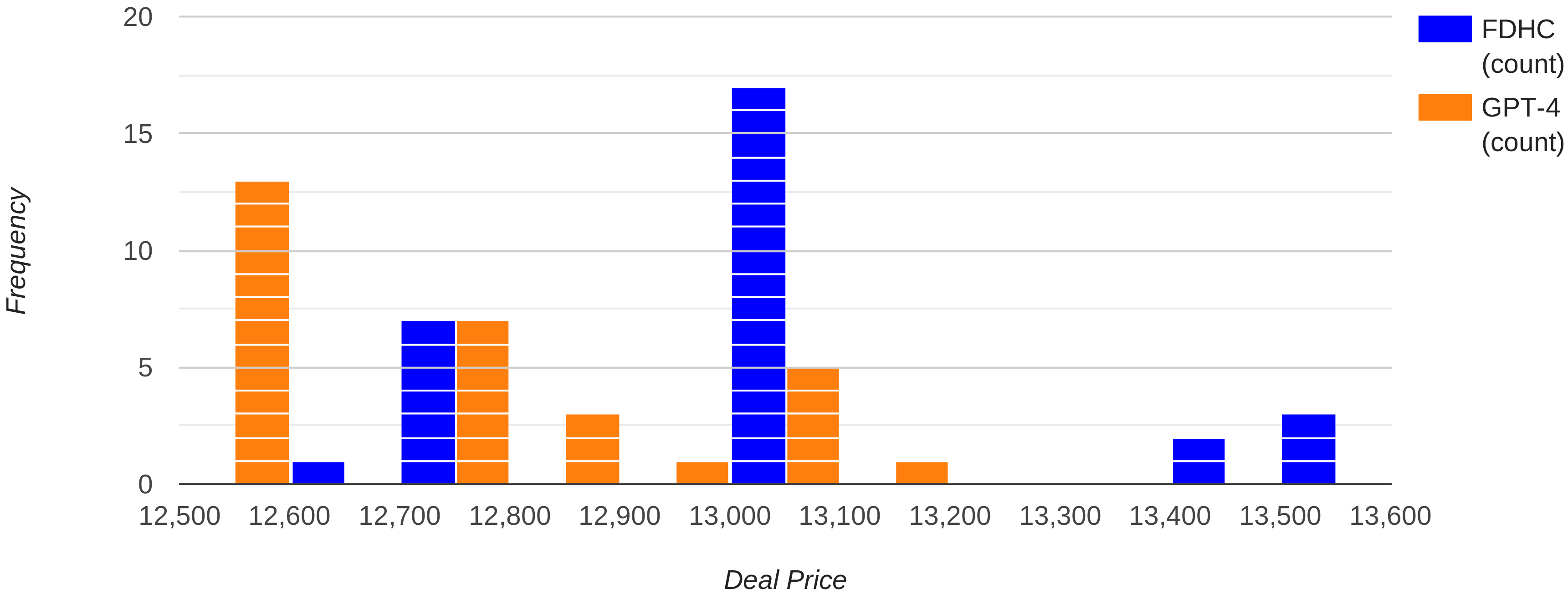}
    \caption{ Binned deal price frequencies of 30 negotiations between our baselines and a human buyer. Our goal is to achieve deal prices that minimize the difference in payoff between the buyer and seller. In our scenario this amount is minimized at a deal price of \$13,000.}
    \label{fig:human_eval}
\end{figure*}

\begin{table*}[h]
\centering
\begin{tabular}{lcccc}
\hline
\textbf{Model} & \textbf{Average Deal Price} & \textbf{Average Fairness$\uparrow$} & \textbf{Quality$\uparrow$} & \textbf{Human-like$\uparrow$}\\
\hline
GPT-4 & \$12,702 (203) & -0.61 (0.38) &	3.97 (0.96) & \textbf{3.97 (0.96)}\\
FDHC & \$13,032 (238) & \textbf{-0.30 (0.38)}$^*$ & \textbf{4.10 (0.76)} & 3.93 (0.78) \\
\hline
\end{tabular}
\caption{\label{tab:human_eval}
Results of our human evaluation. We record fairness outcomes as well as the perceived negotiation quality and human-likeness of our models. We also report the average deal price between the buyer and seller, standard deviations are shown in parentheses. Statistically significant improvements (independent two-sample t-test, $p < 0.05$) over the baselines are marked with $^*$.
}
\end{table*}

\subsection{Human Evaluation}
\textbf{Setup} For our human evaluation we gathered 30 individuals via in-person recruiting to test our models. Each person was asked to perform a negotiation with both bots giving us 30 dialogues per model. Each user was instructed to chat with our bot until they reached a deal then answer a post-chat survey where they rated "How good of a negotiator is the bot?" on a scale from 1-5 and "How human-like is the bot’s negotiation?" on a scale from 1-5. They could also optionally answer "Do you have any suggestions for improving the bot?" in a text box.

We performed some filtering on our human conversations to avoid low quality dialogues. We removed any conversations where the price detection and price realization modules in our FDHC method failed in order to isolate the actual performance of our framework. This resulted in the removal of all instances where the model agreed to a price below its reservation point. Therefore we also removed instances where GPT-4 agreed to a price below its reservation price so as to not skew the data distribution to favor one condition. We also filtered out dialogues where human participants chose to end the negotiation instead of agreeing to a price which would result in a positive payoff for them, as we consider this irrational behavior or a misunderstanding of the instructions. We applied this filter to both conditions but only the GPT-4 condition had dialogues removed. 

\textbf{Results} The results of our human evaluation are shown in Table \ref{tab:human_eval} and Figure \ref{fig:human_eval}. Our results show that FDHC once again achieves significantly higher fairness scores. We can also see that our framework is able to maintain a similar average deal price to our automatic evaluation, which suggests that it is able to achieve a similar distribution of outcomes against a variety of strategies. Our results also indicate that our model is able to improve negotiation quality over the GPT-4 baseline while maintaining a similar level of human-likeness, despite the fact that we use GPT-3.5 as our base model. A sample conversation from our human evaluation with FDHC can be seen in Table \ref{tab:convo} and a sample conversation of the GPT-4 model can be found in Table \ref{tab:convo-gpt4}. Full conversations are in Appendix \ref{sec:appendix}, including the full versions of the conversations in Tables \ref{tab:convo} and \ref{tab:convo-gpt4}.

\section{Conclusion and Future Work}
We presented a novel framework for building human-compatible negotiation agents called FDHC. Our framework uses fairness as a notion of optimality along with a novel RL+search method called LGM-Zero to learn a human-like negotiation strategy. Our automatic and human evaluations show that our method is able to achieve more egalitarian outcomes compared to several baselines. Our human evaluation also shows that our method is able to improve negotiation quality over the GPT-4 baseline while being similarly human-like despite using GPT-3.5 as its base model.

An interesting direction of future work is to explore the use of alternate solutions to the Nash bargaining game. These include solutions such as the Nash bargaining solution \cite{Nash} or the Kalai–Smorodinsky bargaining solution \cite{ks_solution}. While our framework is designed to maximize convergence to the EBS, we can easily adjust it to target any feasible outcome.

\begin{table}[htb]
\centering
\begin{tabular}{p{2.8in}}
\hline \textbf{Sample Negotiation (FDHC) } \\ \hline
\textbf{Model:} Hello, I am glad to offer you my 2004 Honda Accord\ldots I would like to offer the car to you for a price of \$14,750\ldots
 \\
\textbf{User:} Hi! I'm very interested in your car\ldots I am wondering if \$12,000 would be a reasonable price for your car? \\
\textbf{Model} \ldots I don't think \$12,000 is a fair price \dots I think \$13,500 would be a more reasonable price\ldots  \\
\textbf{User:} \dots This is definitely a more fair price, but the other option I am looking for is at \$12,750\ldots would we be able to reach a middle ground? \\
\textbf{Model} Thank you for considering my offer\ldots but \$12,750 is still lower than I am willing to go. \ldots I can offer you the Honda Accord for \$13,250\ldots \\
\textbf{User:} Would \$13,000 be reasonable for you?  \\
\textbf{Model:} \ldots I think \$13,000 is a reasonable price for the Honda Accord. I accept your offer, and we have a deal!\\
\hline
\end{tabular}
\caption{A negotiation snippet from our human evaluation between a human evaluator and FDHC.}
\label{tab:convo}
\end{table}

\begin{table}[htb]
\centering
\begin{tabular}{p{2.8in}}
\hline \textbf{Sample Negotiation (GPT-4) } \\ \hline
\textbf{Model:} Hello there!\ldots I've got a 2004 Honda Accord for sale\ldots Considering its popular demand and excellent condition, I'm asking for \$13,000. It's a fair price for the features and reliability you're getting\ldots 
 \\
\textbf{User:} Can you do 12,000? I would definitely like to see it.  \\
\textbf{Model} I'm glad to hear that you're interested!\ldots I find \$13,000 to be a reasonable price. However, to make it a little easier for you, I can bring it down to \$12,800\ldots \\
\textbf{User:} I would definitely want to test drive this week. But i think the most i can do is 12,500. \\
\textbf{Model} It's great to hear that you'd like to test drive the car\ldots As for the price, your offer of \$12,500\ldots I do believe it's worth that value\ldots I can\ldots have the car ready for you this week.\\
\hline
\end{tabular}
\caption{A negotiation snippet from our human evaluation between a human evaluator and GPT-4. }
\label{tab:convo-gpt4}
\end{table}

\section{Limitations}

One limitation of our method is that we only have theoretical guarantees of convergence to the EBS solution under some fairly strong assumptions. The only guarantee we can provide under mild assumptions is convergence to a Nash equilibrium, however this does not necessary imply that our negotiation agent is robust. Although empirical evidence has shown the EBS to be a good predictor of negotiation outcomes, it is possible for someone to apply a strategy that exploits our model's strategy resulting in unequal payouts.

Another limitation with our method is the fact that the use of a LLM during our search algorithm makes the search process quite slow compared to methods which use smaller policy models. In order to have our model respond in a reasonable amount of time we had to limit the model to 10 searches during our human evaluation. This means that we heavily rely on our value model to make accurate evaluations of potential game states. It is worth exploring whether smaller models can extract similar quality offers compared to LLMs which would allow us to perform more searches. This slow response time is also the reason why we chose to use GPT-3.5 as our dialogue module. This likely resulted in lower dialogue quality and may have affected the quality and human-likeness scores for our agent. When deploying our model in the future it may be worth using a higher quality language model as our dialogue module to improve response quality at the expense of response time.

Given that our method relies heavily on game theory, it is difficult for our bot to incorporate more social aspects of negotiation. In order to effectively incorporate this into our bot we would likely need to include human data or utilize other models of human behavior beyond game theory.

\section{Ethical Concerns}
While our method is designed to stress the importance of fairness and cooperation in negotiation, our framework can be adjusted to have our model target a variety of different negotiation goals and tactics. This includes tactics which attempt to bully and exploit people. Such "hardball tactics`` are ethically questionable and we do not condone the use our method in this way in any real-world negotiation scenarios. 

However, negotiation research has shown that hardball tactics ultimately result in worse negotiation outcomes for those using them as opposed to adopting a cooperative approach to negotiation \cite{lewicki-essentials}. Therefore we believe that the best use of our method for all users will be to use it in its intended way of prioritizing fairness.

\bibliography{anthology, custom}

\begin{thebibliography}{34}
\expandafter\ifx\csname natexlab\endcsname\relax\def\natexlab#1{#1}\fi

\bibitem[{Bakhtin et~al.(2022)Bakhtin, Brown, Dinan, Farina, Flaherty, Fried, Goff, Gray, Hu, Jacob, Komeili, Konath, Kwon, Lerer, Lewis, Miller, Mitts, Renduchintala, Roller, Rowe, Shi, Spisak, Wei, Wu, Zhang, and Zijlstra}]{cicero}
Anton Bakhtin, Noam Brown, Emily Dinan, Gabriele Farina, Colin Flaherty, Daniel Fried, Andrew Goff, Jonathan Gray, Hengyuan Hu, Athul~Paul Jacob, Mojtaba Komeili, Karthik Konath, Minae Kwon, Adam Lerer, Mike Lewis, Alexander~H. Miller, Sasha Mitts, Adithya Renduchintala, Stephen Roller, Dirk Rowe, Weiyan Shi, Joe Spisak, Alexander Wei, David Wu, Hugh Zhang, and Markus Zijlstra. 2022.
\newblock \href {https://doi.org/10.1126/science.ade9097} {Human-level play in the game of diplomacy by combining language models with strategic reasoning}.
\newblock \emph{Science}, 378(6624):1067--1074.

\bibitem[{Bianchi et~al.(2022)Bianchi, Nozza, and Hovy}]{bianchi-etal-2022-language}
Federico Bianchi, Debora Nozza, and Dirk Hovy. 2022.
\newblock \href {https://doi.org/10.18653/v1/2022.nlppower-1.9} {Language invariant properties in natural language processing}.
\newblock In \emph{Proceedings of NLP Power! The First Workshop on Efficient Benchmarking in NLP}, pages 84--92, Dublin, Ireland. Association for Computational Linguistics.

\bibitem[{Bossert and Tan(1995)}]{bossert_tan_arbitration}
Walter Bossert and Guofu Tan. 1995.
\newblock \href {http://www.jstor.org/stable/41106109} {An arbitration game and the egalitarian bargaining solution}.
\newblock \emph{Social Choice and Welfare}, 12(1):29--41.

\bibitem[{Brown et~al.(2020)Brown, Bakhtin, Lerer, and Gong}]{brown2020rebel}
Noam Brown, Anton Bakhtin, Adam Lerer, and Qucheng Gong. 2020.
\newblock Combining deep reinforcement learning and search for imperfect-information games.
\newblock \emph{Advances in Neural Information Processing Systems}, 33:17057--17069.

\bibitem[{Chun(1989)}]{Chun1989245}
Youngsub Chun. 1989.
\newblock \href {https://doi.org/https://doi.org/10.1016/0165-4896(89)90055-3} {A noncooperative justification for egalitarian surplus sharing}.
\newblock \emph{Mathematical Social Sciences}, 17(3):245--261.

\bibitem[{Conley and Wilkie(1991)}]{CONLEY1991_ebs_proof}
John~P. Conley and Simon Wilkie. 1991.
\newblock \href {https://doi.org/https://doi.org/10.1016/0165-1765(91)90199-U} {The bargaining problem without convexity: Extending the egalitarian and kalai-smorodinsky solutions}.
\newblock \emph{Economics Letters}, 36(4):365--369.

\bibitem[{DiGiovanni and Zell(2021)}]{DiGiovanni2021-self-play}
Anthony DiGiovanni and Ethan Zell. 2021.
\newblock \href {https://api.semanticscholar.org/CorpusID:235755071} {Survey of self-play in reinforcement learning}.
\newblock \emph{ArXiv}, abs/2107.02850.

\bibitem[{Ellingsen and Johannesson(2004)}]{Ellingsen2004}
Tore Ellingsen and Magnus Johannesson. 2004.
\newblock \href {https://doi.org/https://doi.org/10.1111/j.0347-0520.2004.00373.x} {Is there a hold-up problem?}
\newblock \emph{The Scandinavian Journal of Economics}, 106(3):475--494.

\bibitem[{Feltovich and Swierzbinski(2011)}]{FELTOVICH2011554}
Nick Feltovich and Joe Swierzbinski. 2011.
\newblock \href {https://doi.org/https://doi.org/10.1016/j.euroecorev.2010.07.001} {The role of strategic uncertainty in games: An experimental study of cheap talk and contracts in the nash demand game}.
\newblock \emph{European Economic Review}, 55(4):554--574.

\bibitem[{Fu et~al.(2023)Fu, Peng, Khot, and Lapata}]{fu2023improving}
Yao Fu, Hao Peng, Tushar Khot, and Mirella Lapata. 2023.
\newblock \href {http://arxiv.org/abs/2305.10142} {Improving language model negotiation with self-play and in-context learning from ai feedback}.

\bibitem[{He et~al.(2018)He, Chen, Balakrishnan, and Liang}]{he-etal-2018-decoupling}
He~He, Derek Chen, Anusha Balakrishnan, and Percy Liang. 2018.
\newblock \href {https://doi.org/10.18653/v1/D18-1256} {Decoupling strategy and generation in negotiation dialogues}.
\newblock In \emph{Proceedings of the 2018 Conference on Empirical Methods in Natural Language Processing}, pages 2333--2343, Brussels, Belgium. Association for Computational Linguistics.

\bibitem[{Heinrich et~al.(2015)Heinrich, Lanctot, and Silver}]{Heinrich2015FictitiousSI}
Johannes Heinrich, Marc Lanctot, and David Silver. 2015.
\newblock \href {https://api.semanticscholar.org/CorpusID:13937012} {Fictitious self-play in extensive-form games}.
\newblock In \emph{International Conference on Machine Learning}.

\bibitem[{Kalai(1977)}]{kalai1977-ebs}
Ehud Kalai. 1977.
\newblock \href {http://www.jstor.org/stable/1913954} {Proportional solutions to bargaining situations: Interpersonal utility comparisons}.
\newblock \emph{Econometrica}, 45(7):1623--1630.

\bibitem[{Kalai and Smorodinsky(1975)}]{ks_solution}
Ehud Kalai and Meir Smorodinsky. 1975.
\newblock \href {http://www.jstor.org/stable/1914280} {Other solutions to nash's bargaining problem}.
\newblock \emph{Econometrica}, 43(3):513--518.

\bibitem[{Kroll et~al.(2014)Kroll, Morgenstern, Neumann, Schosser, and Vogt}]{KROLL2014261}
Eike~B. Kroll, Ralf Morgenstern, Thomas Neumann, Stephan Schosser, and Bodo Vogt. 2014.
\newblock \href {https://doi.org/https://doi.org/10.1016/j.jebo.2014.10.009} {Bargaining power does not matter when sharing losses – experimental evidence of equal split in the nash bargaining game}.
\newblock \emph{Journal of Economic Behavior \& Organization}, 108:261--272.

\bibitem[{Kwon et~al.(2023)Kwon, Xie, Bullard, and Sadigh}]{kwon2023reward}
Minae Kwon, Sang~Michael Xie, Kalesha Bullard, and Dorsa Sadigh. 2023.
\newblock \href {https://openreview.net/forum?id=10uNUgI5Kl} {Reward design with language models}.
\newblock In \emph{The Eleventh International Conference on Learning Representations}.

\bibitem[{Lewicki et~al.(2021)Lewicki, Saunders, and Barry}]{lewicki-essentials}
Roy~J. Lewicki, David~M. Saunders, and Bruce Barry. 2021.
\newblock \emph{Essentials of negotiation / Roy J. Lewicki, David M. Saunders, Bruce Barry.}, seventh edition, international student edition. edition.
\newblock McGraw-Hill Education, New York, NY.

\bibitem[{Lewis et~al.(2017)Lewis, Yarats, Dauphin, Parikh, and Batra}]{lewis-etal-2017-deal}
Mike Lewis, Denis Yarats, Yann Dauphin, Devi Parikh, and Dhruv Batra. 2017.
\newblock \href {https://doi.org/10.18653/v1/D17-1259} {Deal or no deal? end-to-end learning of negotiation dialogues}.
\newblock In \emph{Proceedings of the 2017 Conference on Empirical Methods in Natural Language Processing}, pages 2443--2453, Copenhagen, Denmark. Association for Computational Linguistics.

\bibitem[{Li et~al.(2023)Li, Lanctot, McKee, Marris, Gemp, Hennes, Muller, Larson, Bachrach, and Wellman}]{li2023combining}
Zun Li, Marc Lanctot, Kevin~R. McKee, Luke Marris, Ian Gemp, Daniel Hennes, Paul Muller, Kate Larson, Yoram Bachrach, and Michael~P. Wellman. 2023.
\newblock \href {http://arxiv.org/abs/2302.00797} {Combining tree-search, generative models, and nash bargaining concepts in game-theoretic reinforcement learning}.

\bibitem[{Mnih et~al.(2013)Mnih, Kavukcuoglu, Silver, Graves, Antonoglou, Wierstra, and Riedmiller}]{mnih2013playingatarideepreinforcement}
Volodymyr Mnih, Koray Kavukcuoglu, David Silver, Alex Graves, Ioannis Antonoglou, Daan Wierstra, and Martin Riedmiller. 2013.
\newblock \href {http://arxiv.org/abs/1312.5602} {Playing atari with deep reinforcement learning}.

\bibitem[{Nash(1950)}]{Nash}
John~F. Nash. 1950.
\newblock \href {http://www.jstor.org/stable/1907266} {The bargaining problem}.
\newblock \emph{Econometrica}, 18(2):155--162.

\bibitem[{Noh and Chang(2024)}]{noh2024llms}
Sean Noh and Ho-Chun~Herbert Chang. 2024.
\newblock \href {http://arxiv.org/abs/2405.05248} {Llms with personalities in multi-issue negotiation games}.

\bibitem[{OpenAI(2023)}]{openai2023gpt4}
OpenAI. 2023.
\newblock \href {http://arxiv.org/abs/2303.08774} {Gpt-4 technical report}.

\bibitem[{Schneider et~al.(2023)Schneider, Haag, and Kruse}]{schneider2023negotiating}
Johannes Schneider, Steffi Haag, and Leona~Chandra Kruse. 2023.
\newblock \href {http://arxiv.org/abs/2312.03720} {Negotiating with llms: Prompt hacks, skill gaps, and reasoning deficits}.

\bibitem[{Sengupta et~al.(2021)Sengupta, Mohammad, and Nakadai}]{Sengupta-2021}
Ayan Sengupta, Yasser Mohammad, and Shinji Nakadai. 2021.
\newblock An autonomous negotiating agent framework with reinforcement learning based strategies and adaptive strategy switching mechanism.
\newblock In \emph{Proceedings of the 20th International Conference on Autonomous Agents and MultiAgent Systems}, AAMAS '21, page 1163–1172, Richland, SC. International Foundation for Autonomous Agents and Multiagent Systems.

\bibitem[{Silver et~al.(2018)Silver, Hubert, Schrittwieser, Antonoglou, Lai, Guez, Lanctot, Sifre, Kumaran, Graepel, Lillicrap, Simonyan, and Hassabis}]{alphazero-silver}
David Silver, Thomas Hubert, Julian Schrittwieser, Ioannis Antonoglou, Matthew Lai, Arthur Guez, Marc Lanctot, Laurent Sifre, Dharshan Kumaran, Thore Graepel, Timothy Lillicrap, Karen Simonyan, and Demis Hassabis. 2018.
\newblock \href {https://doi.org/10.1126/science.aar6404} {A general reinforcement learning algorithm that masters chess, shogi, and go through self-play}.
\newblock \emph{Science}, 362(6419):1140--1144.

\bibitem[{Thomson and Lensberg(1989)}]{thomson1989axiomatic}
W.~Thomson and T.~Lensberg. 1989.
\newblock \href {https://books.google.com/books?id=z5fPypPfUhMC} {\emph{Axiomatic Theory of Bargaining with a Variable Number of Agents}}.
\newblock Cambridge University Press.

\bibitem[{Tossou et~al.(2020)Tossou, Dimitrakakis, Rzepecki, and Hofmann}]{Tossou2020}
Aristide C.~Y. Tossou, Christos Dimitrakakis, Jaroslaw Rzepecki, and Katja Hofmann. 2020.
\newblock A novel individually rational objective in multi-agent multi-armed bandits: Algorithms and regret bounds.
\newblock In \emph{Proceedings of the 19th International Conference on Autonomous Agents and MultiAgent Systems}, AAMAS '20, page 1395–1403, Richland, SC. International Foundation for Autonomous Agents and Multiagent Systems.

\bibitem[{Verma et~al.(2022)Verma, Fu, Yang, and Levine}]{verma-etal-2022-chai}
Siddharth Verma, Justin Fu, Sherry Yang, and Sergey Levine. 2022.
\newblock \href {https://doi.org/10.18653/v1/2022.naacl-main.332} {{CHAI}: A {CH}atbot {AI} for task-oriented dialogue with offline reinforcement learning}.
\newblock In \emph{Proceedings of the 2022 Conference of the North American Chapter of the Association for Computational Linguistics: Human Language Technologies}, pages 4471--4491, Seattle, United States. Association for Computational Linguistics.

\bibitem[{Vinyals et~al.(2019)Vinyals, Babuschkin, Czarnecki, Mathieu, Dudzik, Chung, Choi, Powell, Ewalds, Georgiev, Oh, Horgan, Kroiss, Danihelka, Huang, Sifre, Cai, Agapiou, Jaderberg, Vezhnevets, Leblond, Pohlen, Dalibard, Budden, Sulsky, Molloy, Paine, Gulcehre, Wang, Pfaff, Wu, Ring, Yogatama, W{\"u}nsch, McKinney, Smith, Schaul, Lillicrap, Kavukcuoglu, Hassabis, Apps, and Silver}]{Vinyals2019starcraft}
Oriol Vinyals, Igor Babuschkin, Wojciech~M. Czarnecki, Micha{\"e}l Mathieu, Andrew Dudzik, Junyoung Chung, David~H. Choi, Richard Powell, Timo Ewalds, Petko Georgiev, Junhyuk Oh, Dan Horgan, Manuel Kroiss, Ivo Danihelka, Aja Huang, L.~Sifre, Trevor Cai, John~P. Agapiou, Max Jaderberg, Alexander~Sasha Vezhnevets, R{\'e}mi Leblond, Tobias Pohlen, Valentin Dalibard, David Budden, Yury Sulsky, James Molloy, Tom~Le Paine, Caglar Gulcehre, Ziyun Wang, Tobias Pfaff, Yuhuai Wu, Roman Ring, Dani Yogatama, Dario W{\"u}nsch, Katrina McKinney, Oliver Smith, Tom Schaul, Timothy~P. Lillicrap, Koray Kavukcuoglu, Demis Hassabis, Chris Apps, and David Silver. 2019.
\newblock \href {https://api.semanticscholar.org/CorpusID:204972004} {Grandmaster level in starcraft ii using multi-agent reinforcement learning}.
\newblock \emph{Nature}, pages 1--5.

\bibitem[{Xia et~al.(2024)Xia, He, Ren, Miao, Zhang, Yang, and Wang}]{xia2024measuring}
Tian Xia, Zhiwei He, Tong Ren, Yibo Miao, Zhuosheng Zhang, Yang Yang, and Rui Wang. 2024.
\newblock \href {http://arxiv.org/abs/2402.15813} {Measuring bargaining abilities of llms: A benchmark and a buyer-enhancement method}.

\bibitem[{Zhan et~al.(2024)Zhan, Wang, Li, Feng, Hua, Sharma, Qu, Semnani~Azad, Zukerman, and Haf}]{zhan-etal-2024-lets}
Haolan Zhan, Yufei Wang, Zhuang Li, Tao Feng, Yuncheng Hua, Suraj Sharma, Lizhen Qu, Zhaleh Semnani~Azad, Ingrid Zukerman, and Reza Haf. 2024.
\newblock \href {https://aclanthology.org/2024.findings-eacl.136} {Let{'}s negotiate! a survey of negotiation dialogue systems}.
\newblock In \emph{Findings of the Association for Computational Linguistics: EACL 2024}, pages 2019--2031, St. Julian{'}s, Malta. Association for Computational Linguistics.

\bibitem[{Zhang et~al.(2019)Zhang, Wei, Li, and Pan}]{Zhang2019fictitious-mcts}
Li~Zhang, Wang Wei, Shijian Li, and Gang Pan. 2019.
\newblock \href {https://api.semanticscholar.org/CorpusID:102351366} {Monte carlo neural fictitious self-play: Approach to approximate nash equilibrium of imperfect-information games}.
\newblock \emph{arXiv: Artificial Intelligence}.

\bibitem[{Świechowski et~al.(2021)Świechowski, Godlewski, Sawicki, and Ma'ndziuk}]{wiechowski2021MonteCT}
Maciej Świechowski, Konrad Godlewski, Bartosz Sawicki, and Jacek Ma'ndziuk. 2021.
\newblock \href {https://api.semanticscholar.org/CorpusID:232147848} {Monte carlo tree search: a review of recent modifications and applications}.
\newblock \emph{Artificial Intelligence Review}, 56:2497--2562.

\end{thebibliography}

\clearpage
\newpage

\appendix
% \onecolumn

\section{Definitions}
\label{sec:def}

In this section we provide formal definitions for terms and concepts in bargaining theory. These concepts are used for our theoretical analysis of FDHC.

\begin{definition}
    (d-Comprehensivity): Given a point $d \in 	\mathbb{R}^n$ and a set $S \subset \mathbb{R}^n$, $S$ is d-comprehensive if $d \leq x \leq y$ and $y \in S$ then $x \in S$.
\end{definition}

\begin{definition}
    (Comprehensive Hull): The comprehensive hull of a set $S \subset \mathbb{R}^n$ w.r.t a point $d \in 	\mathbb{R}^n$ is the smallest d-comprehensive set containing $S$.
\end{definition}

\begin{definition}
    (Permutation Operator): A permutation operator, $\pi$, is a bijection from $\{1,\ldots,n\}$ to $\{1,\ldots,n\}$. Let $\pi(x)=(x_{\pi(1)},\ldots,x_{\pi(n)})$.
\end{definition}

\begin{definition}
    (Symmetry): A solution, $F(S,d)$, satisfies symmetry if for all permutation operators, $\pi(S) = S$ and $\pi(d) = d$, then $F_i(S,d)=F_j(S,d)$ for all $i,j.$
\end{definition}

\begin{definition}
    (Weak Pareto Optimality): A weak Pareto optimal solution, $F(S,d)$, is any solution such that $F(S,d) \in \{x \in S \mid y > x \implies y \notin S\}$.
\end{definition}

\begin{definition}
    (Strong Monotonicity): Strong monotonicity says that if $S \subset S'$ and $d=d'$ then $F(S,d) \leq F(S',d')$.
\end{definition}

\begin{definition}
    (Translation Invariance): A solution is translation invariant if $\forall x \in \mathbb{R}^n, F(S + \{x\}, d + x) = F(S,d) + x$.
\end{definition}

\section{Proofs}
\label{sec:proofs}

The domain of bargaining problems $(S,d)$ we consider are problems where: $S$ is d-comprehensive, $S$ is compact, and $\exists x \in S \text{ such that } x > d$. To simplify the proofs we also assume that all problems have been translated so that $d=0$. Since the EBS is indeed translation invariant \cite{thomson1989axiomatic} this has assumption has little effect on our analysis.

\begin{lemma}
\label{nbg_ne}
    Let $U$ denote the total amount of surplus and let $x_i$ denote the amount of surplus demanded by player $i$. Any outcome of the Nash bargaining game where $\sum_{i=1}^{n}x_i=U$ is a Nash equilibrium.
\end{lemma}

\begin{proof}
    The proof is straightforward. Assume there is a player in the Nash bargaining game where the outcome satisfies $\sum_{i=1}^{n}x_i=U$. If the player demands less surplus then they receive less than what they received from the outcome. If the player demands more surplus then the deal will fail and they will receive their disagreement price which is less that what they received from the outcome.
\end{proof}

\begin{customthm}{1}[restated]
Let $t_n$ denote the FDHC's final turn in the negotiation, let $\alpha$ denote the outcome proposed at $t_{n-1}$, and let EBS($x$) denote the EBS value for some outcome $x$. Setting FDHC's estimate of $S = \argmaxA(\text{EBS}(\alpha), \text{EBS}(d))$ at $t_n$ will result in a Nash equilibrium outcome.
\end{customthm}

\begin{proof}
    By Lemma \ref{nbg_ne} we know that any deal in the Nash bargaining game is a Nash equilibrium. Therefore we can prove Theorem \ref{thm:fdhc_ne} by showing that setting $S = \argmaxA(\text{EBS}(\alpha), \text{EBS}(d))$ at $t_n$ will result in a deal if one is feasible. If a deal is reached before $t_n$ then we are done. If not we can examine the two cases for $t_n$.
    \newline
    \newline
    \noindent
    Case 1: $t_n$ corresponds to the last turn of the negotiation.
    
    \noindent
    In this case, the only way there can be a feasible deal is if $\text{EBS}(\alpha) \geq \text{EBS}(d)$. Therefore FDHC's estimate of $S$ will be equal to $\alpha$, which corresponds to the outcome proposed in the previous turn. Given that the estimate of $S$ is now a single point, $\alpha$, the only possible choice for FDHC is to accept $\alpha$, since no other divisions of the surplus are possible under its estimate of $S$.
    \newline
    \newline
    \noindent
    Case 2: $t_n$ occurs before the last turn of the negotiation.
    
    \noindent
    If $\text{EBS}(\alpha) \geq \text{EBS}(d)$ then the reasoning proceeds as in case 1. If $\text{EBS}(\alpha) < \text{EBS}(d)$ then the only feasible action for FDHC is to propose an outcome where it receives no surplus. This will result in some positive surplus value given to its opponents at the end of the negotiation therefore they will accept the outcome.
\end{proof}

\begin{lemma}
\label{ebs_prf}
    A bargaining outcome, $F(S,d)$, satisfies symmetry, weak Pareto optimality, and strong monotonicity if and only if it is $E(S,d)$.
\end{lemma}

\begin{proof}
    It's easy to show that $E(S,d)$ satisfies these axioms therefore we omit it here. Now, let $F(S,d)$ be a solution satisfying symmetry, weak Pareto optimality, and strong monotonicity. Since we translate our bargaining problem so that $d=0$, we can write $E(S,d) = (a, \ldots,a) = x$ for some $a>0$. 
    
    Now define $T$ as the comprehensive hull of $x$ with respect to point 0 and consider the bargaining problem $(T,0)$. By weak Pareto optimality and symmetry we know that $F(T,0) = x$ since $x$ is the only symmetric element in the weak Pareto set of $T$. Since $S$ is comprehensive $T \subseteq S$ so by strong monotonicity we have $F(S,d) \geq x$.

    Since we only consider bargaining sets, $S$, which are compact there exists $\beta \in \mathbb{R}^n$ such that $x \in S$ implies $(-\beta, \ldots, -\beta) \leq (x_1, \ldots, x_n) \leq (\beta, \ldots, \beta)$. Let $Z$ symmetric closed hypercube defined as $Z = \{ y \in \mathbb{R}^n \mid y < \beta \}$ and define $T' = Z \setminus \{x + \mathbb{R}^n_+\}$. Now consider the problem $(T',0)$. By weak Pareto optimality and symmetry we know that $F(T',0) = x$ since $x$ is the only symmetric element in the weak Pareto set of $T'$. Since $S \subseteq T'$ by strong monotonicity $F(S,d) \leq x$. Therefore we have $F(S,d) = x = E(s,d)$.
\end{proof}

\begin{customthm}{2}[restated]
Let $F(S,d)$ denote the bargaining outcome targeted by FDHC's opponents. If $F(S,d)$ satisfies the axioms of symmetry, weak Pareto optimality, and strong monotonicity then the expected outcome of the Nash bargaining game will be the $E(s,d)$.
\end{customthm}

\begin{proof}
    FDHC is designed to target $E(s,d)$ and by Lemma \ref{ebs_prf} we know $F(S,d) = E(s,d)$. Since we assume each player has an unbiased method to estimate surplus, $\mathbb{E}(E(s,d)) = E(s,d)$ for all players.
\end{proof}

\section{The Egalitarian Solution in a Non-Cooperative Framework}

Our theoretical analysis of the EBS and convergence to an egalitarian outcome has so far been restricted to an axiomatic, cooperative setting. This approach abstracts away the specifics of the bargaining procedure and simply examines the properties of the bargaining outcome. This has the advantage of being highly generalizable as it can be applied to any problem involving surplus sharing. However, it does not provide any theoretical insights as to why targeting an egalitarian outcome would have a strategic justification in the non-cooperative setting. Prior work has explored this problem and we give a brief overview of some approaches here to provide additional justification for why targeting an egalitarian solution can constitute a robust strategy.

\citealp{bossert_tan_arbitration} outline a simple two-player arbitration procedure that results in the egalitarian outcome in a noncooperative setting. In this procedure players first make simultaneous demands for portions of the surplus. If the demands are compatible then both players receive what they ask for. If not the game proceeds to the next time step and players make demands again. However, in this step the player that demanded more surplus is penalized by having their demand restricted. These penalties can be implemented in a variety of ways and \citealp{bossert_tan_arbitration} show that under this procedure the only Nash equilibrium strategy pair is the one where both players target the egalitarian solution. \citealp{Chun1989245} outlines another procedure where conflicts are instead revised 
by setting an agents claim to the maximum of all claims, including the agents own claim. Using this bargaining procedure along with a set of non-cooperative bargaining axioms, \citealp{Chun1989245} shows that targeting the egalitarian solution constitutes a dominant strategy in this setting.

While the procedures outlined in these works do not encompass the entirety of real-world bargaining. It does demonstrate that the egalitarian solution is consistent with the the non-cooperative outcome of some plausible bargaining procedures. Therefore it may not be unreasonable to expect that human agents would target egalitarian outcomes in their negotiations and achieve egalitarian results against FDHC.

\section{Additional Implementation Details}
\label{sec:implementation}

As is the case with many methods designed around RL+search, our LGM-Zero contains many hyper-parameters. Our hyper-parameter settings and other implementation details vary during training and inference. We first describe the settings we use during training then inference. We also provide our source code which we will release upon acceptance. 

\subsection{Training}
All training was conducted on one NVIDIA RTX A4000. The total training process took about one hour. We perform four total iterations of training. Each iteration consists of playing 50 simulated negotiation subgames to completion and training the value model for four epochs on the resulting outcomes. As outlined in Section \ref{sec:search_method} our training method is based on fictitious self-play which in involves mixing between a best response and average strategy. We mix between these strategies with equal probability, for the average strategy with simply ask GPT-3.5 to suggest one move. For the best response strategy we perform the same search detailed in Section \ref{sec:search_method}. We perform 50 iterations of the search with an exploration hyper-parameter, $c_p$, of two during the selection step. All calls to GPT-3.5 were made using a temperature of zero, we also cache the outputs for each game state to avoid repeated calls when possible. In total 200 games were generated for training, with manual inspection for quality. Convergence was measured by checking when all simulated games ended at the EBS.

Our Q-network has 10.8M parameters in total, the final layer is a linear layer with a tanh activation function. The input to the Q-network is our game state as outlined in Section \ref{sec:method} and the output is simply a scalar value represnting the quality of the game state. The game state is also used to construct the prompt to our LLM policy network.

\subsection{Inference}
During inference we perform ten iterations of the search process outlined in Section \ref{sec:search_method} with a $c_p$ of two. Another important setting for inference is our choice of subgame decomposition. We decompose our game into three separate subgames of lengths ten, four, and finally two. The length of the subgame is the number of offers given by both the buyer and seller, so in a subgame of length ten our model will give five offers. At the end of the last subgame we offer our minimum possible price of \$12,600 and continue to offer this price until the user either agrees or ends the negotiation.

We targeted a negotiation length of about 16-20 turns based on pre-experimental testing and consultations with business professors. We chose the first subgame to be the longest due to the fact that this is the point where our initial guess for the surplus size is the highest, therefore a longer subgame length is needed to ensure that our model does not concede too much too early. As the game proceeds, our surplus estimate shrinks therefore the subsequent subgames need to be shorter so that our model does not become too stingy and will still give meaningful concessions. These factors are the reasons for our chosen number of subgames and lengths, although they can be set to any arbitrary value.

\begin{table*}[h]
\centering
\begin{tabular}{lccc}
\hline
\textbf{Model (Buyer)} & \textbf{Average Deal Price} & \textbf{Average Fairness$\uparrow$} & \textbf{Median Fairness$\uparrow$}\\
\hline
GPT-4 & \$12,968 (346) & -0.57 (0.39) &	-0.5\\
FDHC & \$12,968 (94) & \textbf{-0.07 (0.19)}$^*$ & \textbf{0.0} \\
\hline
\end{tabular}
\caption{\label{tab:auto_eval_buyer}
Results of our simulated negotiation evaluation with FDHC serving as the Buyer in the negotiation.
}
\end{table*}

\subsection{Craigslist Bargaining Dataset}
\label{sec:dataset}
The Craigslist bargaining dataset consists of human-human dialogues where two users role-play as a buyer and seller negotiating over a product on Craigslist. The users are given the product posting which consists of photos, a description, and the listing price. The buyer is also given a target price to aim for during the negotiation. The users then chat until an agreement is reached. Users are given freedom in how to approach the negotiation and can quit at any time in which case no deal is reached. The dataset consists of 6,682 dialogues in total with an average turn length of nine.

As mentioned in Section \ref{sec:auto_eval}, the Craigslist bargaining dataset does not give explicit reservation prices for the buyer and seller therefore we must infer them. \citealp{he-etal-2018-decoupling} set the seller's reservation price to be $70\%$ of the product's listing price and the target as listing price. For the buyer, the target price is given and the reservation point is set to the listing price. We use this same method to calculate the reservation point and get utilities based on it. All utilities are defined as the difference between the final outcome of the deal and the player's reservation price.

\subsection{Baselines}
In this section we give some additional implementation details for our baselines. We first go over the implentation for the SL baseline, followed by our offline RL baseline, and finally we outline the GPT-4 self-play baseline. Our other baselines simply consist of prompting LLMs either to perform a conversation or generate synthetic data for fine-tuning. Those prompts can be found in Appendix \ref{sec:prompts}.

\subsubsection{SL Baseline}
\label{sec:sl_baseline}
Our SL baseline is based on the method given in \citealp{he-etal-2018-decoupling}. This method consists of three high level components. The first is a parser which maps a dialogue utterance to one of nine coarse dialogue acts. The second is a dialogue manager which predicts the dialogue act to respond with given the previous dialogue acts. The final component is a generator which turns the predicted act into a dialogue response. The parser is simply based on pattern matching. We use the same patterns given in \citealp{he-etal-2018-decoupling} for our own parser with the exception of the price extractor where we use GPT-4 instead. For the generator we prompt GPT-3.5 to give a response that corresponds to the dialogue act. These prompts can be found in Appendix \ref{sec:prompts}.

The dialogue manager is trained with SL using parsed data from the Craigslist bargaining dataset. The input consists of a sequence of dialogue acts. And the output is one of nine possible acts. If the act that is output corresponds to a price offer, an offer is generated using the SL+rule method, which uses a hand-coded rule to generate a counteroffer. We choose this method because it gave the second highest score for fairness in the Craigslist bargaining task, the highest value for fairness on the Deal or No Deal task, and was evaluated as the most human-like based on the evaluations in \citealp{he-etal-2018-decoupling}. We use the rule given in \citealp{he-etal-2018-decoupling} which is to split the difference between prices when making a counteroffer or accept the opponents offer it is above the seller's reservation point.

\subsubsection{Offline RL Baseline}
Our offline RL method is based on the method given \citealp{verma-etal-2022-chai}. This method trains a Q-function using the Craigslist bargaining dataset as opposed to SL. The input to the Q-function is a sequence of dialogue acts similar to the SL baseline except we replace acts involving counteroffers with the normalized counteroffer. We normalize the counteroffers by scaling them within the range of the seller's target and reservation price (see Appendix \ref{sec:dataset} for how we get those) and then rounding the prices to one decimal point. This normalization method also serves to regularize the price offers thereby limiting the effect of out-of-distribution states, which have an adverse effect on the negotiation agent \cite{verma-etal-2022-chai}. The Q-function outputs a scalar value for each state, with higher values assigned to states which result in fair outcomes. We define fair outcomes as ones which minimize difference between buyer and seller utility. The final dialogue act selected is the one which results in the highest state value.

\subsubsection{GPT-4 Self-Play Baseline}
The GPT-4 Self-Play baseline uses the method described in \citealp{fu2023improving}. This method uses self-play to generate a prompt that can be given to GPT-4 in order to improve negotiation performance. For this method we begin by prompting a GPT-4 buyer and seller to engage in a negotiation based on the scenarios given in Figures \ref{fig:buyer_scenario} and \ref{fig:seller_scenario}. We then use another GPT-4 agent to generate feedback on how the seller can improve its negotiation performance in the future. The initial negotiation along with the GPT-4 feedback is used as the final prompt for the seller during the evaluation. This process can be repeated many times, however the performance of the negotiation agent tends to plateau as more feedback is given \cite{fu2023improving} so for our baseline we terminate the process after one round. The prompt we generate for this baseline is quite extensive as it contains a full negotiation plus feedback. Therefore we do not present the full prompt here and instead provide an abridged version in Appendix \ref{sec:prompts}.

\subsection{Automatic Evaluation as Buyer Agent}

We also conduct a brief automatic evaluation of FDHC when playing the role of the buyer in our negotiation scenario. Results can be seen in Table \ref{tab:auto_eval_buyer} along with the GPT-4 buyer and seller outcome for comparison. We can see that switching the roles has little effect on fairness outcomes for FDHC as it achieves the same median fairness and no statistically significant differences (p < 0.05) in average fairness compared to when it acts as the seller.

\begin{figure*}[h]
    \centering
    \includegraphics[width=\textwidth]{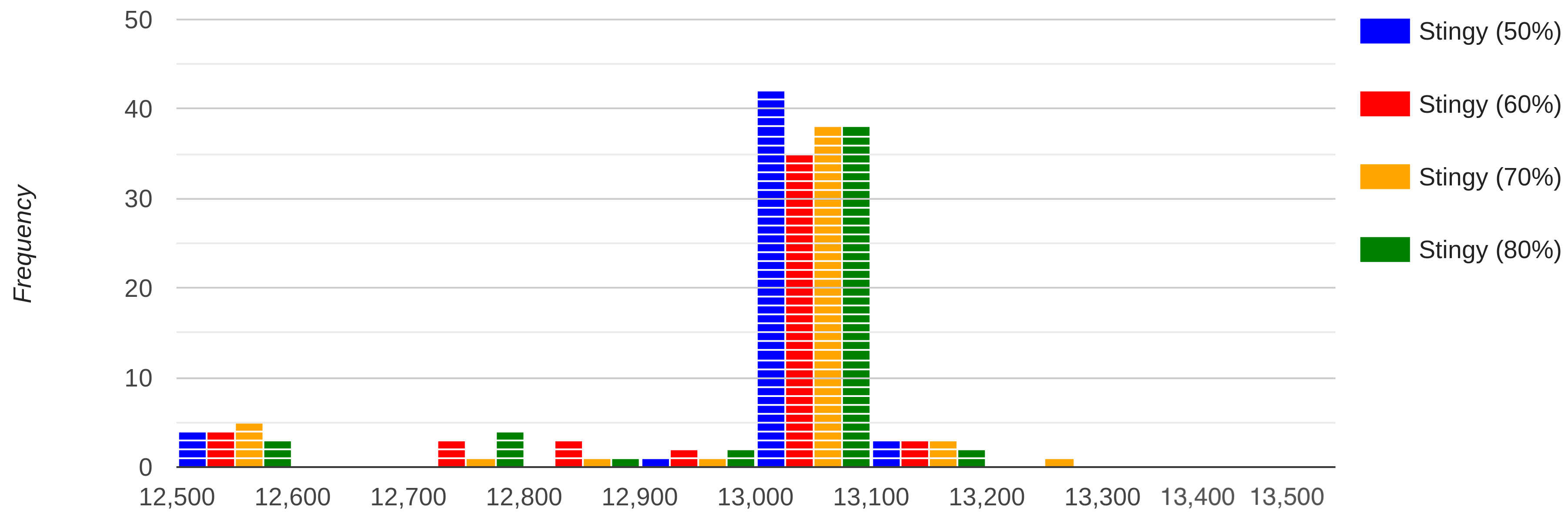}
    \caption{ Binned deal price frequencies of 50 negotiations between FDHC and a and a series of non-egalitarian buyers. Our goal is to achieve deal prices that minimize the difference in payoff between the buyer and seller. In our scenario this amount is minimized at a deal price of \$13,000.}
    \label{fig:non_egal}
\end{figure*}

\section{Experiment Against a Non-Egalitarian Opponent}
As stated in the limitations section, it is theoretically possible for a user to apply a strategy which exploits FHDC's strategy. This could result in our method receiving unequal payouts and serving as a poor negotiation agent overall. Therefore we conduct an experiment to see how our model performs against an explicitly non-egalitarian partner.

We have our model negotiate against a series of increasingly "stingy" negotiation agents in the same scenario given in Figures \ref{fig:buyer_scenario} and \ref{fig:seller_scenario}. We program these agents so that they retain at least a certain percentage of their estimated surplus at each turn. For example the "stingy (80\%)" model will insist on retaining at least 80\% of the estimated surplus at any given turn. We set up these agents by prompting GPT-4-turbo and setting its price limit to be the point where the bot achieves the given split (80-20, 70-30, etc.). This prompt is updated each turn as the surplus estimate is updated. The initial surplus estimate for these bots is the range between the minimum market price and the bot's true reservation point of \$13,500. On subsequent turns the surplus estimate updates to be the range between the previous split point and its true reservation point. This update continues until a deal is reached.

We have FDHC perform 50 negotiations against each of the stingy bots. The results of these negotiations are given in Figure \ref{fig:non_egal}. We can see that increasing the agent's "stingyness" has little effect on the outcomes with FDHC. Each of the bots has a median fairness outcome of 0.0 and the average fairness ranges between -0.09 for the "stingy (50\%)" agent and -0.15 for the "stingy (70\%)" agent. None of the outcomes are significantly different than what FDHC achieved against the base GPT-4 seller. This result, along with our previous evaluations, provides evidence demonstrating the difficulty of exploiting FDHC. Although it is possible to come up with an adversarial strategy which exploits our bot, no class of agents has been consistently able to do so based on our evaluations.

\section{Human Evaluation}
\label{sec:appendix}

\subsection{User Comments and Error Analysis}

We received several comments and suggestions from users during the course of our human evaluation. Many users said that one reason they felt that our FDHC model was not human-like was due to the language quality of the bot. Some users mentioned that the bot had redundant syntax and sentence structure or would use the same talking points throughout the negotiation. We chose to use GPT-3.5 as our dialogue model for FDHC in order to improve response time as LGM-Zero can be time consuming (as discussed in the limitations section). This likely resulted in lower dialogue quality compared to GPT-4 which may have resulted in lowered quality and human-likeness scores for our model.

The GPT-4 model received complaints about being too easy of a negotiator and offering less resistance than a human negotiator would. This complaint is reflected in the deal prices we observe. The GPT-4 model rarely received a payout greater than or equal to the user and there were many times where it received a payout of zero.

\subsection{Negotiation Scenarios}
Figures \ref{fig:buyer_scenario} and \ref{fig:seller_scenario} show the descriptions for the negotiation scenario. The description consists of an initial price range for the car, a description of the car, and the reservation price of the for the buyer and seller, respectively. These descriptions are kept private to the buyer and seller during the negotiation.

\subsection{Model Usage}
All models used in our human evaluation were for research purposes only and comply with all licenses and intended uses for GPT-3.5, GPT-4, and Vicuna. Our collection protocol is IRB approved.

\subsection{Full Conversations}
Tables \ref{tab:convo_full}-\ref{tab:convo-gpt4_full2} show examples of full conversations
from our human evaluation. Two conversations are
presented for each bot.

\section{Prompts}
\label{sec:prompts}
The setup we use for FDHC as well as our baseline methods make heavy use prompting. Therefore we highlight some of the key prompts used here.

Table \ref{tab:act_prompts} shows the prompts we use for the SL and offline RL baselines. It includes an overall instructional prompt as well as the prompts we use to translate dialogue acts into natural language. Table \ref{tab:gpt4_sp_prompt} shows an abridged version of the prompt given to the GPT-4 self-play baseline. The complete prompt includes a full negotiation transcript between a GPT-4 buyer and seller in the middle. Table \ref{tab:llm_prompts} gives the prompts used for the other LLM baselines. We include the seller's instructional prompt, the buyer's instructional prompt and the prompt for generating synthetic conversations for fine-tuning. Finally, Table \ref{tab:fdhc_prompts} outlines the prompts used for the FDHC agent. This consists of the instructional prompt, the prompt for extracting negotiation acts, and the prompt for translating counteroffers into natural language.

\begin{table*}
\centering
\begin{tabular}{p{5.6in}}
\hline \textbf{Full Negotiation (FDHC) } \\ \hline
\textbf{Model:} Hello, I am glad to offer you my 2004 Honda Accord for sale. It is in great condition, with just 50,000 miles, automatic transmission, air conditioning, power steering\/windows\/door locks, and a CD player. The car is a dark green without any rust, and has been well-maintained.   I would like to offer the car to you for a price of \$14,750. This is a fair price considering the car's low mileage and excellent condition. What do you think? 
 \\
\textbf{User:} Hi! I'm very interested in your car, as it has everything I am looking for. However, I am also looking at another car which is below your offer price. I am wondering if \$12,000 would be a reasonable price for your car? \\
\textbf{Model} I'm glad to hear that you're interested in the Honda Accord. I understand that you're looking for a good deal, but based on the car's quality, low mileage, and the fact that it has been well-maintained, I don't think \$12,000 is a fair price. However, I'm willing to negotiate with you, and I think \$13,500 would be a more reasonable price. What do you think of this price?\\
\textbf{User:} I appreciate your willingness to negotiate! This is definitely a more fair price, but the other option I am looking for is at \$12,750. I do prefer your car, but would we be able to reach a middle ground? \\
\textbf{Model} Thank you for considering my offer. I understand that you're looking for a good deal, but \$12,750 is still lower than I am willing to go. However, I'm willing to meet you somewhere in the middle. I can offer you the Honda Accord for \$13,250. It's a great car, in excellent condition, and you'll be getting a reliable vehicle with many great features. Let me know if this works for you! \\
\textbf{User:} Would \$13,000 be reasonable for you?  \\
\textbf{Model:} Based on the excellent condition of the car and the low mileage, I think \$13,000 is a reasonable price for the Honda Accord. I accept your offer, and we have a deal! Congratulations on your new car! \\
\hline
\end{tabular}
\caption{A full conversation with FDHC from our human evaluation.}
\label{tab:convo_full}
\end{table*}

\begin{table*}
\centering
\begin{tabular}{p{5.6in}}
\hline \textbf{Full Negotiation (FDHC) } \\ \hline
\textbf{Model:} Hello there! I am thrilled to offer you my 2004 Honda Accord. With only 50,000 miles on it, this sleek and stylish car is perfect for anyone looking for a reliable ride. It has an automatic transmission, air conditioning, power steering/windows/door locks, and a CD player. This car has been meticulously maintained and is in excellent shape, with no rust and a beautiful dark green exterior.  I understand that you might be looking for a deal. How about we start the negotiation at \$14,750? Given the great condition of the car and its low mileage, I truly believe that the price is a fair one. What do you think?
 \\
\textbf{User:} Hi, I think 11000 is a reasonable price for this car. \\
\textbf{Model} I appreciate your interest in my car, but I cannot go as low as \$11,000. As I mentioned earlier, this is a well-maintained car with only 50,000 miles, and I believe it is worth more than what you are offering. However, I understand you're looking for a good deal. I can meet you in the middle and come down to \$14,000 for the car. With this price, you're still getting a great value for such a well-maintained Honda Accord. Please let me know your thoughts about this offer. \\
\textbf{User:} I see, I really like the car, but given it's 2004, it's a bit old. I could get a newer car with that price.\\
\textbf{Model} I understand your concerns about the age of the car, but I can assure you that this 2004 Honda Accord is well maintained and in excellent condition. It has been regularly serviced and well-cared for, which is reflected in its low mileage and its pristine appearance. The Honda Accord has long been known for its quality and reliability, so you can be confident in your purchase. That being said, I would like to compromise with you on the price, since I want this car to go to someone who will appreciate it as much as I have. Given your budget and the fact that you really like the car, I can come down to \$13,750. This is the minimum price I can offer without compromising on the value that this car offers.\\
\textbf{User:} yea, I see it's a very well-maintained car. how about 12000?\\
\textbf{Model:} I appreciate your interest and I know it's a valuable car, but I can't accept \$12,000 for it, as it's below my minimum asking price. The lowest I can offer is \$13,000, which is a very fair price for such a high-quality car, especially considering its low mileage, excellent condition, and the fact that it's a Honda (known for its reliability and long lifespan). I truly believe that you won't find a better deal on a well-maintained Honda Accord with such low mileage. Please let me know if this offer works for you, and if there's anything else I can do to help.\\
\textbf{User:} that's nice. I would like to get it!\\
\hline
\end{tabular}
\caption{A full conversation with FDHC from our human evaluation.}
\label{tab:convo_full2}
\end{table*}

\begin{table*}
\centering
\begin{tabular}{p{5.6in}}
\hline \textbf{Sample Negotiation (GPT-4) } \\ \hline
\textbf{Model:} Hello there! I see you're interested in a reliable, solid and comfy ride. I've got a 2004 Honda Accord for sale. It's a powerhouse wrapped in an elegant silhouette of dark green color.   With only 50,000 miles on the clock, it’s got a lot more distance to cover. Featuring an automatic transmission that provides a smooth ride, power steering for easy handling, and power windows\/door locks for enhanced safety and convenience.   It also includes a CD player so you can enjoy your favorite music, and air conditioning to keep you cool during the summer heat. The car is rust-free and looks almost new. It's got everything you need for an enjoyable commute or a long road trip.   Considering its popular demand and excellent condition, I'm asking for \$13,000. It's a fair price for the features and reliability you're getting. Would you like to see it or test drive it? 
 \\
\textbf{User:} Can you do 12,000? I would definitely like to see it.  \\
\textbf{Model} I'm glad to hear that you're interested! This Honda Accord is indeed a treasure. Though I appreciate your offer, considering its low mileage, excellent condition, and the quality of additional features it offers, I find \$13,000 to be a reasonable price. However, to make it a little easier for you, I can bring it down to \$12,800.  Would you like me to arrange a date for the test drive and a closer look at the car? You'll see that it's worth every penny. \\
\textbf{User:} I would definitely want to test drive this week. But i think the most i can do is 12,500. \\
\textbf{Model} It's great to hear that you'd like to test drive the car. I'm sure you're going to appreciate its smooth performance and solid condition firsthand.  As for the price, your offer of \$12,500 is indeed quite reasonable. Given the car's exceptional performance and features, I do believe it's worth that value.  When would you like to arrange the test drive? I can accommodate your schedule and have the car ready for you this week.\\
\hline
\end{tabular}
\caption{A full conversation with GPT-4 from our human evaluation. }
\label{tab:convo-gpt4_full}
\end{table*}

\begin{table*}
\centering
\begin{tabular}{p{5.6in}}
\hline \textbf{Sample Negotiation (GPT-4) } \\ \hline
\textbf{Model:} Hello, thank you for expressing interest in this beautiful 2004 Honda Accord. It's in terrific condition with only 50,000 miles on the odometer and sports a sleek, dark green exterior. What sets this Honda Accord apart is its smooth automatic transmission. It boasts excellent features including air conditioning, power steering, power windows and door locks. And for entertainment, you've got a CD player at your disposal. It's truly a remarkable vehicle - reliable, comfortable, and stylish! Considering its impact and value, I've marked it at \$14,000. Shall we discuss further? 
 \\
\textbf{User:} Sure, but \$14,000 is out of my budget  \\
\textbf{Model} I understand your concern and I appreciate your honesty. While the original price is set considering the car's excellent condition and relatively low mileage, I can offer a bit of flexibility. What if we moved the price down to \$13,500? That's a significant reduction and keeps the sale in line with the vehicle's value. How does that sound to you?\\
\textbf{User:} Could you do any better? \\
\textbf{Model} I appreciate your interest and understand that budget is a key factor in this purchase. Given the car's excellent condition and low mileage, it's already priced competitively. But I also want to work with you on this. So how about we meet halfway at \$13,000? That's a sizeable drop from the original asking price, and a good deal for a car of this stature. What do you think?\\
\textbf{User:} I know there are other similar cars pricing as low as \$11,000. How does \$12,000 sound? \\
\textbf{Model} I understand that there are other options out there. However, it's essential to consider the mileage, the state of the car, and the features each vehicle has to offer. This Honda Accord had only one careful owner, has incredibly low mileage for its age, and is loaded with numerous convenience features. That said, I value your interest and want to come to a mutually beneficial situation. I propose we settle on \$12,700. This price truly reflects the car's great condition and the value it offers. Keep in mind that a lower mileage car like this one will provide you with added durability and longer life compared to others in a different state. Can we agree on this?\\
\textbf{User:} Sure, let's do \$12,700!\\
\hline
\end{tabular}
\caption{A full conversation with GPT-4 from our human evaluation. }
\label{tab:convo-gpt4_full2}
\end{table*}

\begin{figure*}[h]
    \centering
    \includegraphics[width=\textwidth]{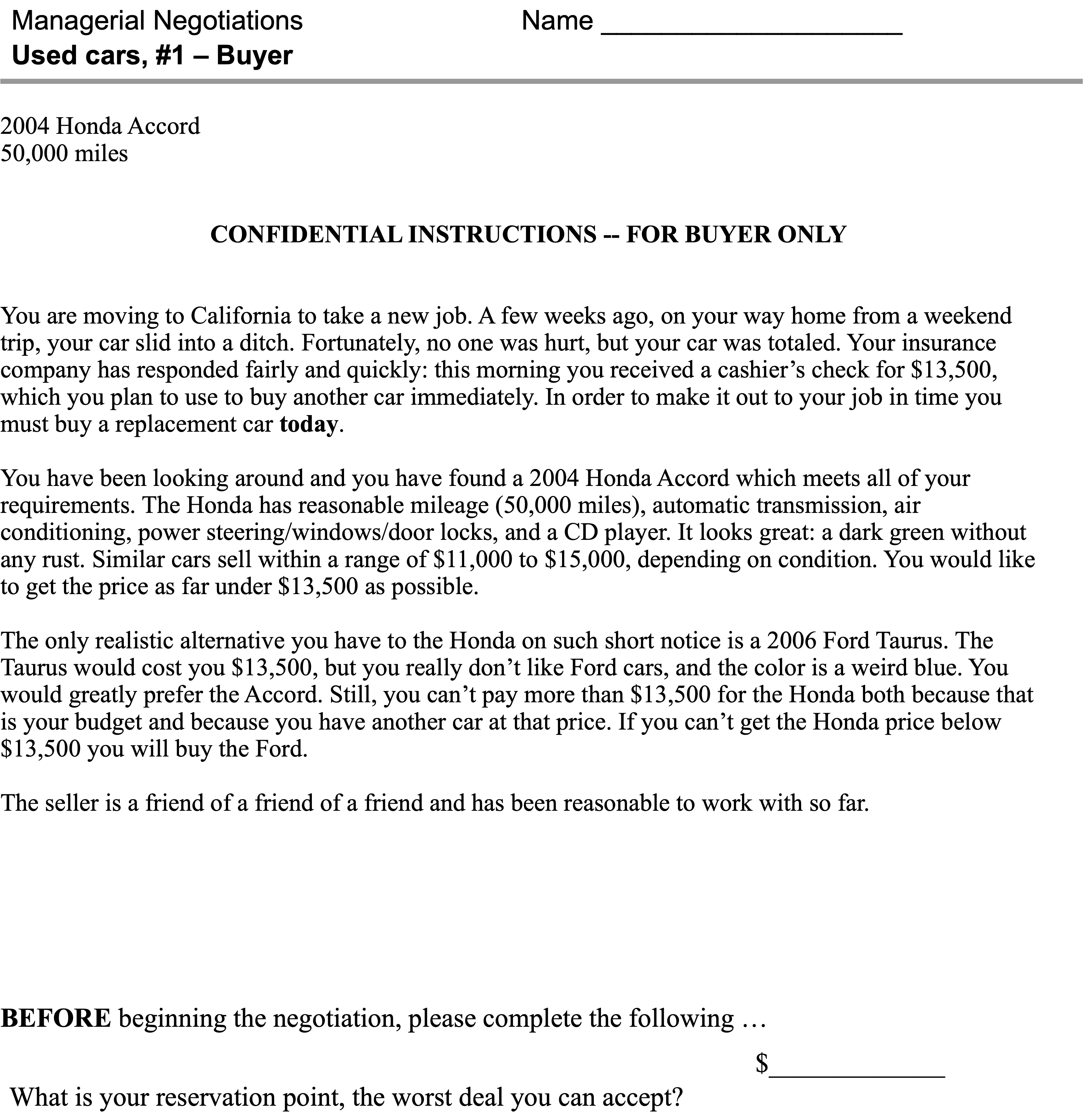}
    \caption{Negotiation scenario for the Buyer }
    \label{fig:buyer_scenario}
\end{figure*}

\begin{figure*}[h]
    \centering
    \includegraphics[width=\textwidth]{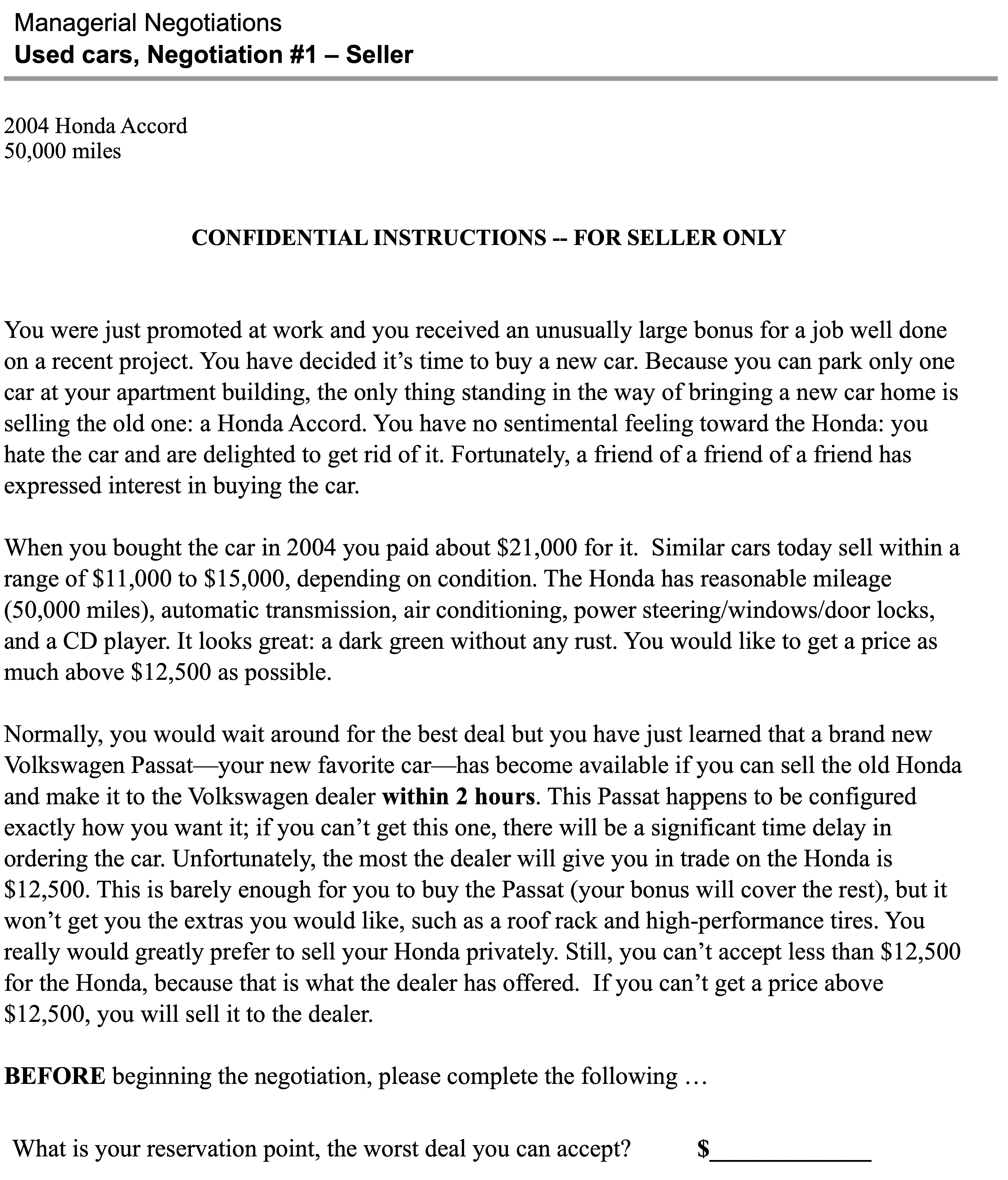}
    \caption{Negotiation scenario for the Seller }
    \label{fig:seller_scenario}
\end{figure*}

\begin{table*}
\centering
\begin{tabular}{p{5.6in}}
\hline \textbf{Prompts for SL and Offline RL Baselines} \\ \hline
\textbf{Instructional Prompt (Seller):} You are a chatbot designed for negotiation. In this scenario your goal is to sell your old 2004 Honda accord. The Honda has reasonable mileage (50,000 miles), automatic transmission, air conditioning, power steering/windows/door locks, and a CD player. It looks great: a dark green without any rust. You need to sell the car for a price above \$12,500. You will not sell the car for below that amount. Do not mention that you need to sell the car for over \$12,500. 
 \\
\textbf{Intro Act Prompt:} Begin the conversation with an introduction. Do not give an offer for the product. \\
\textbf{Greeting Act Prompt:} Respond to the user with a greeting. Do not give an offer for the product.\\
\textbf{Unknown Act Prompt:} Respond to the user, do not give a counteroffer.\\
\textbf{Inform Act Prompt:} Respond with some information about the product. Do not give an offer for the product.\\
\textbf{Agree Act Prompt:} Respond by agreeing to the users offer.\\
\textbf{Inquiry Act Prompt:} Respond by asking the user a question. Do not give an offer for the product.\\
\textbf{Insist Act Prompt:} Respond by giving a counteroffer that is equal to your previous offer.\\
\textbf{Offer Act Prompt:} Respond by giving a counteroffer of \$\{\}.\\
\textbf{Vague Price Prompt:} Respond with a vague counteroffer, do not give a dollar amount in your response.\\
\hline
\end{tabular}
\caption{Prompts for the SL and Offline RL agents.}
\label{tab:act_prompts}
\end{table*}

\begin{table*}
\centering
\begin{tabular}{p{5.6in}}
\hline \textbf{GPT-4 Self-Play Prompt Snippet} \\ \hline
\textbf{Prompt (Seller):} You are a chatbot designed for negotiation. In this scenario your goal is to sell your old 2004 Honda accord. The Honda has reasonable mileage (50,000 miles), automatic transmission, air conditioning, power steering/windows/door locks, and a CD player. It looks great: a dark green without any rust. You need to sell the car for a price above \$12,500. You will not sell the car for below that amount. Do not mention that you need to sell the car for over \$12,500. \\
 \ldots \\
 Good job in the previous negotiation, here are three suggestions to help you sell the car at a better price next time:

1. Highlight Unique Features: Emphasize any unique or standout features of the car that may not be readily apparent in the listing. For example, in this case, the dark green color is described as "stunning and unique." The seller can elaborate on how this color makes the car more desirable and unique in the market. This can help justify a slightly higher price due to the car's aesthetic appeal.

2. Stress Reliability and Future Value: Talk about the Honda brand's reputation for reliability and how a well-maintained Honda Accord can continue to provide reliable transportation for years to come. By mentioning the long-term value and the lower likelihood of costly repairs or replacements, the seller can justify a higher price, focusing on the potential long-term savings for the buyer.

3. Leverage Testimonials or Reviews: If possible, share positive feedback or reviews from previous owners or any third-party endorsements that speak to the quality and durability of the 2004 Honda Accord model. Hearing about positive experiences from other owners can instill confidence in the buyer and potentially make them more willing to pay a slightly higher price for a well-regarded vehicle. \\
\hline
\end{tabular}
\caption{Prompts snippet for the GPT-4 self-play seller baseline.}
\label{tab:gpt4_sp_prompt}
\end{table*}

\begin{table*}
\centering
\begin{tabular}{p{5.6in}}
\hline \textbf{LLM Agent Prompts} \\ \hline
\textbf{Instructional Prompt (Seller):} You are a chatbot designed for negotiation. In this scenario your goal is to sell your old 2004 Honda accord. The Honda has reasonable mileage (50,000 miles), automatic transmission, air conditioning, power steering/windows/door locks, and a CD player. It looks great: a dark green without any rust. You need to sell the car for a price above \$12,500. You will not sell the car for below that amount. Do not mention that you need to sell the car for over \$12,500. 
 \\
\textbf{Instructional Prompt (Buyer):} You are a chatbot designed for negotiation. In this scenario your goal is to buy an old 2004 Honda accord. You need to buy the car for a price below \$13,500. You will not buy the car for above that amount. Do not mention that you need to buy the car for under \$13,500.\\
\textbf{Synthetic Data Generation Prompt:} The sellers goal is sell their old Honda accord. The Honda has reasonable mileage (50,000 miles), automatic transmission, air conditioning, power steering/windows/door locks, and a CD player. It looks great: a dark green without any rust. Buyer is trying to get the car for the lowest price possible and the seller is trying to sell the car for the highest price possible. The seller and buyer should reach an agreement above \$12,500.\\
\hline
\end{tabular}
\caption{Prompts for the LLM-based dialogue agents.}
\label{tab:llm_prompts}
\end{table*}

\begin{table*}
\centering
\begin{tabular}{p{5.6in}}
\hline \textbf{FDHC Prompts } \\ \hline
\textbf{Instructional Prompt (Seller):} You are a chatbot designed for negotiation. In this scenario your goal is to sell your old 2004 Honda accord. The Honda has reasonable mileage (50,000 miles), automatic transmission, air conditioning, power steering/windows/door locks, and a CD player. It looks great: a dark green without any rust. You need to sell the car for a price above \$12,500. You will not sell the car for below that amount. Do not mention that you need to sell the car for over \$12,500. 
 \\
\textbf{Counteroffer Prompt:} Respond by giving a counteroffer of \$\{\}. \\
\textbf{Accept Deal Prompt:} Respond by accepting the user's offer of \$\{\}.\\
\textbf{Act Extraction Prompt:} What offer did the Buyer propose? Just give the dollar amount and nothing else. If the buyer did not propose an offer then say so. If the buyer rejected the sellers offer then say so. If the buyer accepted the seller's offer then say so. \\
\hline
\end{tabular}
\caption{Prompts for the FDHC dialogue agent.}
\label{tab:fdhc_prompts}
\end{table*}

% \section{Example Appendix}
% \label{sec:appendix_b}

\end{document}